\documentclass{article} 
\usepackage{arxiv}

\usepackage{amsmath,amsfonts,bm}

\def\eqref#1{equation~\ref{#1}}

\def\1{\bm{1}}

\DeclareMathAlphabet{\mathsfit}{\encodingdefault}{\sfdefault}{m}{sl}
\SetMathAlphabet{\mathsfit}{bold}{\encodingdefault}{\sfdefault}{bx}{n}

\newcommand{\E}{\mathbb{E}}

\newcommand{\R}{\mathbb{R}}

\DeclareMathOperator*{\argmax}{arg\,max}
\DeclareMathOperator*{\argmin}{arg\,min}

\usepackage{amsthm}
\usepackage{pgfplots}
\usepackage{natbib}
\usepgfplotslibrary{groupplots}
\usepackage{subcaption}
\usepackage{siunitx} 
\pgfplotsset{compat=1.18}
\pgfplotsset{
  myaxis/.style={
    width=0.39\linewidth,
    height=0.28\linewidth,
    label style={font=\scriptsize},
    tick label style={font=\scriptsize},
    title style={font=\footnotesize},
    grid=both,
    legend style={
      at={(0.5,1.15)}, anchor=south, legend columns=3, font=\scriptsize
    }
  }
}
\usepackage{caption}

\usepackage{hyperref}
\usepackage{url}
\usepackage{bbm}
\usepackage{amsmath}
\usepackage{amsfonts}
\usepackage{amssymb}
\usepackage{xcolor}
\usepackage{dsfont}
\newtheorem{definition}{Definition}
\usepackage{natbib}
\usepackage{mathtools}
\renewcommand{\eqref}[1]{(\ref{#1})}
\usepackage{boxedminipage}
\usepackage{empheq}
\newtheorem{assumption}{Assumption}
\newtheorem{theorem}{Theorem}
\newtheorem{remark}{Remark}
\newtheorem{lemma}{Lemma}
\usepackage{algorithm}
\usepackage{algorithmic}
\usepackage[ruled,algo2e,vlined]{algorithm2e}

\definecolor{antiquewhite}{rgb}{0.98, 0.92, 0.84} 
\definecolor{blizzardblue}{rgb}{0.67, 0.9, 0.93}
\title{Neyman-Pearson Classification under Both Null and Alternative Distributions Shift}

\author{
Mohammadreza M. Kalan\\
Univ Rennes, Ensai, CNRS,\\ CREST–UMR 9194,
F-35000 Rennes, France\\
\texttt{mohammadreza.kalan@ensai.fr}
\And
Yuyang Deng\\
Columbia University, Department of Statistics\\
\texttt{yd2824@columbia.edu}
\And
Eitan J. Neugut\\
Columbia University, Department of Statistics\\
\texttt{eitan.neugut@columbia.edu}
\And
Samory Kpotufe\\
Columbia University, Department of Statistics\\
\texttt{samory@columbia.edu}
}

\newcommand{\cbr}[1]{\left\{ #1 \right\}}
\newcommand{\pare}[1]{\left( #1 \right)}
\newcommand{\norm}[1]{\left\| #1 \right\|}
\newcommand{\inprod}[2]{\ensuremath{\langle #1 , \, #2 \rangle}}

\begin{document}

\maketitle

\begin{abstract}
We consider the problem of transfer learning in Neyman–Pearson classification, where the objective is to minimize the error w.r.t. a distribution $\mu_1$, subject to the constraint that the error w.r.t. a distribution $\mu_0$ remains below a prescribed threshold. While transfer learning has been extensively studied in traditional classification, transfer learning in imbalanced classification such as Neyman–Pearson classification has received much less attention. This setting poses unique challenges, as both types of errors must be simultaneously controlled. Existing works address only the case of distribution shift in $\mu_1$, whereas in many practical scenarios shifts may occur in both $\mu_0$ and $\mu_1$. We derive an adaptive procedure that not only guarantees improved Type-I and Type-II errors when the source is informative, but 
also automatically adapt to situations where the source is uninformative, thereby avoiding negative transfer. In addition to such statistical guarantees, the procedures is efficient, 
as shown via complementary computational guarantees.  



\end{abstract}

\section{Introduction}
In many applications, the objective is to learn a decision rule from data that separates two classes with distributions $\mu_0$ and $\mu_1$, whose sample sizes are often imbalanced. These applications include disease diagnosis \citep{myszczynska2020applications,bourzac2014diagnosis}, malware detection in cybersecurity \citep{alamro2023automated,kumar2019edima}, and climate science, such as heavy rain detection \citep{folino2023learning,frame2022deep}. In such settings, the learner must control the error w.r.t. both classes. This imbalance also induces an asymmetry in the importance of the two types of errors. A natural framework to address this asymmetry is Neyman–Pearson (NP) classification \citep{cannon2002learning,scott2005neyman}, where the objective is to minimize the error with respect to $\mu_1$, subject to the constraint that the error with respect to $\mu_0$ does not exceed a pre-specified threshold $\alpha$. To mitigate data scarcity in the primary classification task (referred to as the \emph{target}), where samples from either class may be limited, one can make use of additional datasets (referred to as the \emph{source}). However, while transfer learning has been extensively studied in traditional classification, how to effectively leverage source data in the NP classification has received far less attention.

In this work, we propose a provable procedure that effectively leverages the source under potential distribution shifts in both $\mu_0$ and $\mu_1$ to improve both Type-I error, i.e., the error w.r.t. $\mu_0$, and Type-II error, i.e., the error w.r.t. $\mu_1$. Prior theoretical works \citep{kalan2025transfer,kalantight} address only the special case where the source shares the same $\mu_0$ with the target and the shift occurs solely in $\mu_1$. However, in many applications, shifts may occur in both distributions, raising the challenge of how to exploit the source to improve both types of errors simultaneously. 

A key difficulty in the presence of a shift in $\mu_0$ is that a classifier satisfying the source constraint, i.e., achieving source Type-I error below $\alpha$, does not necessarily satisfy the target constraint. Especially when only a few target samples are available from $\mu_0$, the set of classifiers that meet the empirical $\alpha$ constraint on these samples may still substantially exceed the true population risk w.r.t. $\mu_0$. Moreover, one cannot straightforwardly use the source to rule out classifiers that result in large population error w.r.t. $\mu_0$ in target, since the correspondence between the source constraint and the target constraint is unknown. Thus, the central challenge is how to leverage source samples to control the target Type-I error while simultaneously improving the target Type-II error—an issue that does not arise in the setting considered by prior works \citep{kalan2025transfer,kalantight}.

The procedure proposed in this work adaptively leverages the source without requiring any prior knowledge about its relatedness to the target. When the source is informative, it improves both types of errors on the target; when the source is uninformative, it matches the performance of using only target data, thereby avoiding negative transfer. Moreover, the generalization error bound derived in this work recovers the bounds established in prior works \citep{kalan2025transfer,kalantight} in the special case where the source and target share the same distribution $\mu_0$. The proposed procedure consists of two stages. In the first stage, it utilizes source class-0 samples to determine an effective threshold $\hat{\alpha}_S$ that aligns the source Type-I constraint with the target constraint. This step rules out classifiers that would exceed the target Type-I constraint $\alpha$, while retaining those capable of achieving low Type-II error on the target. In the second stage, it leverages source class-1 samples to further reduce the Type-II error on the target while ensuring that the Type-I error remains below the prescribed threshold.


In addition to a statistical guarantee for the proposed procedure, we establish a computational guarantee by reformulating the learning procedure as a constrained optimization problem. Specifically, we design a computationally feasible procedure that reduces the learning procedure to a sequence of convex programs and leverages a stochastic convex optimization solver as an algorithmic component. We show that this optimization procedure results in a model that achieves the statistical guarantee in polynomial time.
\vspace{-2mm}
\section{Related Work}
Transfer learning has been extensively studied in traditional classification and regression \citep{li2022transfer,cai2021transfer,kpotufe2021marginal,tripuraneni2020theory,mousavi2020minimax,ben2010theory,ben2006analysis,mansour2009domain}. In the context of binary classification over a hypothesis class, \citep{hanneke2019value} introduce the notion of transfer exponent, which translates the performance of a classifier from the source domain to the target domain, and achieve the minimax rate adaptively in balanced classification. However, none of these works address imbalanced classification, where the central challenge is to simultaneously control both Type-I and Type-II errors.

The problem of NP classification when the underlying distributions are unknown except through samples was first formulated by \citet{cannon2002learning, scott2005neyman}, who considered empirical risk minimization with respect to one class while constraining the empirical error on the other below a pre-specified threshold. \citet{rigollet2011neyman} studied the same problem under a surrogate convex loss. \citet{tong2013plug} studied a nonparametric NP classification framework and established rates of convergence for a plug-in approach based on estimating the class distributions. More recently, \citet{kalan2024distribution} derived distribution-free minimax rates for NP classification and showed that, unlike traditional classification, the problem exhibits a dichotomy between fast and slow rates.

Related to this work, in the context of transfer learning in NP classification, \citet{kalantight} derived minimax rates for $0$-$1$ loss. Building on this, \citet{kalan2025transfer} introduced an implementable transfer learning procedure for NP classification using a surrogate loss and derived upper bound guarantees. Both \citet{kalantight,kalan2025transfer} focused on the restricted setting where the source and target share the class-$0$ distribution $\mu_0$ and the shift occurs only in $\mu_1$. Moreover, they established only statistical guarantees without addressing computational aspects. However, in this work, we consider the general setting where distribution shifts may occur in both $\mu_0$ and $\mu_1$. We propose an adaptive procedure with statistical guarantees that also recovers the special case in \citet{kalantight,kalan2025transfer}. On the computational side, we reformulate the learning problem within a two-stage convex programming framework and develop a concrete stochastic optimization procedure that outputs a model with the desired excess-risk guarantee with a bound on the gradient complexity of the procedure.

\vspace{-2mm}
\section{Setup}
Let $(\mathcal{X},\Sigma)$ be a measurable space, and let $\mathcal{H}$ denote a hypothesis class of measurable functions $h:\mathcal{X}\to\mathbb{R}$. In the binary classification setting, each $h\in\mathcal{H}$ induces a decision rule of the form $\mathds{1}\{h(x)\ge 0\}$, which assigns a sample $x$ to either class $0$ or class $1$. In this paper, we adopt the Neyman–Pearson classification framework. Classes $0$ and $1$ are generated according to distributions $\mu_0$ and $\mu_1$, respectively, and the goal is to learn a function $\hat{h}\in\mathcal{H}$ from samples drawn from these distributions. The classifier $\hat{h}$ is required to keep the error w.r.t. class $0$ (i.e., the Type-I error) below a pre-specified threshold $\alpha$, while minimizing the error w.r.t. class $1$ (i.e., the Type-II error). We first introduce a loss function w.r.t. which the Type-I and Type-II errors are defined.
\begin{definition}[Surrogate Loss]\label{def_loss}
We call a function $\varphi:\R\to\R_{+}$ an $L$-Lipschitz surrogate loss whenever the following hold: 
$\varphi$ is monotone nondecreasing and normalized by $\varphi(0)=1$; it satisfies the Lipschitz bound $|\varphi(x)-\varphi(y)|\le L|x-y|$ for all $x,y\in\R$; and there exists some constant $C>0$ such that, for every $h\in\mathcal{H}$ and $x\in\mathcal{X}$, we have  $
\max\{\varphi(h(x)),\,\varphi(-h(x))\}\le C.$
\end{definition}
Next, we define Type-I and Type-II errors w.r.t. a surrogate loss.
\begin{definition}
For a surrogate loss $\varphi$, the $\varphi$-Type-I and $\varphi$-Type-II errors of $h$ are defined as
$R_{\varphi,\mu_0}(h)=\E_{\mu_0}[\varphi(h(X))]$ and
$R_{\varphi,\mu_1}(h)=\E_{\mu_1}[\varphi(-h(X))]$, respectively.
\end{definition}
In particular, for the indicator loss $\varphi(z)=\mathds{1}\{z\ge 0\}$, the errors $R_{\varphi,\mu_0}$ and $R_{\varphi,\mu_1}$ reduce to the standard Type-I and Type-II errors.

We consider a transfer learning setting with source and target domains, where for each $D\in\{S,T\}$ the class $0$ and class $1$ distributions are denoted by $\mu_{0,D}$ and $\mu_{1,D}$. The Neyman–Pearson classification problem w.r.t. a surrogate loss $\varphi$ and a pre-specified threshold $\alpha$ in each domain $D\in\{S,T\}$ is defined as
\begin{equation}\label{def-neyman-pearson-cls}
\min_{h \in \mathcal{H}} \; R_{\varphi,\mu_{1,D}}(h) \quad 
\text{s.t. } R_{\varphi,\mu_{0,D}}(h) \le \alpha
\end{equation}
and we denote by $h^*_{D,\alpha}$ a (not necessarily unique) solution to \eqref{def-neyman-pearson-cls}. \cite{kalan2025transfer,kalantight} consider the special case where $\mu_{0,S}=\mu_{0,T}$, whereas in this paper we address the more general setting in which $\mu_{0,S}$ and $\mu_{0,T}$ may differ. Furthermore, note that when $\mathcal{H}$ contains all measurable functions from $\mathcal{X}$ to $\mathbb{R}$ and $\varphi$ is the indicator loss, the Neyman--Pearson Lemma \citep{lehmann1986testing} characterizes the solution to \eqref{def-neyman-pearson-cls} as $h^*_{D,\alpha} = 2\mathds{1}\!\left\{\tfrac{p_{0,D}(x)}{p_{1,D}(x)} \ge \lambda\right\} - 1$,
for a suitable $\lambda$, under mild regularity conditions, where $p_{0,D}$ and $p_{1,D}$ denote the class-conditional densities of classes $0$ and $1$ in domain $D\in\{S,T\}$.

In a practical setting, the learner does not have access to the distributions $\mu_{0,D}$ and $\mu_{1,D}$ except through observing i.i.d. samples from each class in both domains: $\{X^{(0,D)}_i\}_{i=1}^{n_{0,D}}\!\sim\mu_{0,D}$ and $\{X^{(1,D)}_i\}_{i=1}^{n_{1,D}}\!\sim\mu_{1,D}$ for $D\in\{S,T\}$. Then, the corresponding empirical $\varphi$-Type-I and $\varphi$-Type-II errors for $D\in\{S,T\}$ are
\[
\widehat R_{\varphi,\mu_{0,D}}(h):=\frac{1}{n_{0,D}}\sum_{i=1}^{n_{0,D}}\varphi\!\big(h(X^{(0,D)}_i)\big),
\qquad
\widehat R_{\varphi,\mu_{1,D}}(h):=\frac{1}{n_{1,D}}\sum_{i=1}^{n_{1,D}}\varphi\!\big(-h(X^{(1,D)}_i)\big).
\]

The goal of the learner in a transfer learning setting is to learn a function $\hat{h}\in \mathcal{H}$ using $n_{0,S}, n_{1,S}, n_{0,T}, n_{1,T}$ samples from the source and target domains, such that it performs well on the \emph{target} domain. Specifically, the learner minimizes the target $\varphi$-Type-II excess error
\[
\mathcal{E}_{1,T}(\hat h)
:= \big[ R_{\varphi,\mu_{1,T}}(\hat h) - R_{\varphi,\mu_{1,T}}(h^*_{T,\alpha}) \big]_+,
\qquad [u]_+ := \max\{0,u\},
\]
subject to the $\varphi$-Type-I error constraint $
R_{\varphi,\mu_{0,T}}(\hat{h}) \leq \alpha + \epsilon_{0,T},$
where $\epsilon_{0,T}$ is of order $n_{0,T}^{-1/2}$.

In this paper, we aim to develop an adaptive procedure that, without requiring any prior knowledge about the relatedness between source and target, effectively leverages both types of samples with two guarantees. First, regardless of whether the source is related to the target, it is as good as using only target samples and ignoring the source. Hence, it avoids negative transfer. Second, whenever the source is related to the target, it improves upon the performance of using only target samples. By considering the empirical counter part of \eqref{def-neyman-pearson-cls} in the target domain
\begin{equation}
\begin{aligned}\label{def-neyman-pearson-empirical}
\min_{h \in \mathcal{H}} \; \hat{R}_{\varphi,\mu_{1,T}}(h) \quad 
\text{s.t. } \hat{R}_{\varphi,\mu_{0,T}}(h) \le \alpha+\epsilon_{0,T}
\end{aligned}
\end{equation}
where $\epsilon_{0,T} = \tfrac{\tilde{C}}{\sqrt{n_{0,T}}}$ for a constant $\tilde{C}$ specified later, it can be shown that $\mathcal{E}_{1,T}(\hat h)\lesssim \frac{1}{\sqrt{n_{1,T}}}$ and $R_{\varphi,\mu_{0,T}}(\hat{h})-\alpha \lesssim\frac{1}{\sqrt{n_{0,T}}}$ \citep{cannon2002learning,scott2005neyman}.

In the transfer learning setting, \cite{kalan2025transfer,kalantight} develop adaptive procedures for the special case where $\mu_{0,T}=\mu_{0,S}$. They stablish the bounds $R_{\varphi,\mu_{0,T}}(\hat{h})-\alpha\lesssim \frac{1}{\sqrt{n_{0}}}$, where $n_0$ denotes the number of samples from $\mu_0=\mu_{0,S}=\mu_{0,T}$, and  $$\mathcal{E}_{1,T}(\hat h)\lesssim \min\{\frac{1}{\sqrt{n_{1,T}}},R_{\varphi,\mu_{1,T}}(h^*_{S,\alpha})-R_{\varphi,\mu_{1,T}}(h^*_{T,\alpha})+(\frac{1}{\sqrt{n_{1,S}}})^{1/\rho}\}$$
where $\rho$ is the transfer exponent \citep{hanneke2019value}, which quantifies the relatedness between source and target.

When $\mu_{0,S}\neq \mu_{0,T}$ the feasible set in \eqref{def-neyman-pearson-cls} under the target constraint need not coincide with its source counterpart. Therefore, a classifier that satisfies the source $\varphi$-Type-I at level $\alpha$ need not satisfy the target constraint. More precisely, let $\mathcal{H}_{D}(\alpha)\coloneqq \left\{h\in\mathcal{H} : R_{\varphi,\mu_{0,D}}(h)\leq \alpha \right\}$ for $D\in \{S,T\}$. If $\mu_{0,S}\neq \mu_{0,T}$, then $h\in \mathcal{H}_{S}(\alpha)$ does not imply $h\in \mathcal{H}_{T}(\alpha)$. 

Furthermore, since the learner only observes samples rather than having direct access to the distributions, the empirical set $\hat{\mathcal{H}}_{T}(\alpha):=\left\{h\in\mathcal{H} : \hat{R}_{\varphi,\mu_{0,T}}(h)\leq \alpha+\epsilon_{0,T} \right\}$ may yield a $\varphi$-Type-I error substantially larger than $\alpha$, especially when $n_{0,T}$ is small and consequently $\epsilon_{0,T}$ is large. In addition, one does not know in advance an appropriate value $\alpha'\in[0,1]$ in the source domain such that exploiting $\hat{\mathcal{H}}_{S}(\alpha')=\left\{h\in\mathcal{H} : \hat{R}_{\varphi,\mu_{0,S}}(h)\leq \alpha'+\epsilon_{0,S} \right\}$, where $\epsilon_{0,S}=\frac{\tilde{C}}{\sqrt{n_{0,S}}}$, would lead to improved $\varphi$-Type-I performance in target. Choosing too small an $\alpha'$ in source may reduce the $\varphi$-Type-I error in target, but at the cost of a substantially larger $\varphi$-Type-II error in target.

These challenges motivate the need for an adaptive transfer learning procedure, which we develop in the next section.
\vspace{-2mm}
\section{Adaptive transfer learning procedure}\label{section-TL-Procedure}

We start with defining the smallest $\alpha'\in[0,1]$ such that the constrained set in the source, i.e. $\mathcal{H}_{S}(\alpha')$ contains the optimal function in the target as:
\begin{align}\label{alpha_s_definition}
    \alpha_{S}:=\inf\{\alpha':\mathcal{H}_{S}(\alpha')\cap T^*(\alpha)\neq\emptyset\}
\end{align}
where $T^*(\alpha)\subset\mathcal{H}$  denote the set of solutions of Target problem \eqref{def-neyman-pearson-cls}. Next, we define the set of functions whose empirical $\varphi$-Type-I error in the target satisfies the $\alpha$ constraint, and whose empirical $\varphi$-Type-II error in the target domain does not exceed that of the corresponding empirical risk minimizer. Specifically,  
\begin{align}\label{equation_H_Star}
\hat{\mathcal{H}}^*_{\alpha,T}
:=\Big\{h\in \hat{\mathcal{H}}_T(\alpha):
\hat{R}_{\varphi,\mu_{1,T}}(h)\leq 
\hat{R}_{\varphi,\mu_{1,T}}(\hat{h}^*_{T,\alpha-\epsilon_{0,T}})
+6\epsilon_{1,T}\Big\},
\end{align}
where $\epsilon_{1,T}=\frac{\tilde{C}}{\sqrt{n_{1,T}}}$ and $\hat{h}^*_{T,\alpha-\epsilon_{0,T}}:=\argmin\limits_{h\in \hat{\mathcal{H}}_{T}(\alpha-\epsilon_{0,T})}
\hat{R}_{\varphi,\mu_{1,T}}(h)$ . The population $\varphi$-Type-I error of functions in $\hat{\mathcal{H}}^*_{\alpha,T}$ can substantially exceed $\alpha$ when the number of target samples from $\mu_{0,T}$, i.e., $n_{0,T}$, is small. To address this, we exploit the source domain by restricting it to a suitable subset, thereby retaining only functions with lower $\varphi$-Type-I error. Specifically, we first define
\begin{equation}\label{empirical-threshold}
\hat{\alpha}_S
:= \inf\big\{\alpha'\in[\alpha,1]:\ \hat{\mathcal{H}}_{S}(\alpha')\cap \hat{\mathcal{H}}^{*}_{\alpha,T}\neq \emptyset\big\}.
\end{equation}
and then introduce the restricted set $\hat{\mathcal{H}}'=\hat{\mathcal{H}}_{S}(\hat{\alpha}_S)\cap \hat{\mathcal{H}}^*_{\alpha,T}.$ We will show formally later that the functions in $\hat{\mathcal{H}}'$ satisfy the required bound on the $\varphi$-Type-I error. At this stage, we need to further restrict $\hat{\mathcal{H}}'$ to achieve a small $\varphi$-Type-II error. For this purpose, we introduce the following sets for $D\in\{S,T\}$:
\begin{equation}
\hat{\mathcal{H}}'_{1,D}:=\{h\in\hat{\mathcal{H}}':\hat{R}_{\varphi,\mu_{1,D}}(h)\leq \hat{R}^*_{\varphi,\mu_{1,D}}(\hat{\mathcal{H}}')+2\epsilon_{1,D}\}
\end{equation}
where $\epsilon_{1,D}=\frac{\tilde{C}}{\sqrt{n_{1,D}}}$ and $\hat{R}^*_{\varphi,\mu_{1,D}}(\hat{\mathcal{H}}')=\min\limits_{h\in\hat{\mathcal{H}}'} \hat{R}_{\varphi,\mu_{1,D}}(h).$ Now, equipped with these definitions, we propose the following transfer learning procedure, which outputs a function $\hat{h}\in \mathcal{H}$.

\begin{boxedminipage}{\textwidth}
\begin{align}\label{procedure_intersection}
&\text{If } \hat{\mathcal{H}}'_{1,S}\cap \hat{\mathcal{H}}'_{1,T}\neq \emptyset, \text{then choose } \hat{h}\in \hat{\mathcal{H}}'_{1,S}\cap \hat{\mathcal{H}}'_{1,T}. 
    \nonumber\\[6pt]
&\text{Otherwise, choose } \hat{h}=\argmin_{h\in\hat{\mathcal{H}}'} \hat{R}_{\varphi,\mu_{1,T}}(h).
\end{align}
\end{boxedminipage}

\section{Main Results}
\subsection{Generalization Error Bounds}\label{section_Generalization_bounds}

We begin with the definition of Rademacher complexity, a capacity measure used to ensure uniform convergence guarantees between the empirical and population risks.

\begin{definition}[Rademacher Complexity \citep{bartlett2002rademacher}]
    Let $X_1,\ldots,X_n$ be i.i.d.\ random variables drawn from a distribution $\mu$ on $\mathcal{X}$. Let $\sigma_1,\ldots,\sigma_n$ be independent Rademacher variables, i.e.\ random signs with 
$\Pr(\sigma_i=+1)=\Pr(\sigma_i=-1)=\tfrac{1}{2}$. The empirical Rademacher complexity of $\mathcal{H}$ is defined as
$\widehat{R}_n(\mathcal{H})
= \mathbb{E}_{\sigma}\!\left[ \sup\limits_{h \in \mathcal{H}}
\frac{1}{n}\sum_{i=1}^n \sigma_i h(X_i) \right].$
The Rademacher complexity of $\mathcal{H}$ is then $
R_n(\mathcal{H}) = \mathbb{E}_{X_1^n}\!\left[ \widehat{R}_n(\mathcal{H}) \right].$
\end{definition}
\begin{assumption}[Class Complexity]\label{assump_rademacher}
We assume that the complexity of the hypothesis class satisfies $
R_n(\mathcal{H}) \leq\frac{B_{\mathcal{H}}}{\sqrt{n}}$ 
for some constant $B_{\mathcal{H}}$ capturing the complexity of $\mathcal{H}$.
\end{assumption} 
Note that when the input features are bounded, many hypothesis classes used in practice—such as linear models or neural networks with bounded parameters—satisfy Assumption \ref{assump_rademacher} \citep{golowich2018size}. Furthermore, we make the following convexity assumption on the hypothesis class, which, together with the convexity of the loss, ensures that small deviations in the $\varphi$-Type-I error do not lead to large deviations in the $\varphi$-Type-II error. The bounds can also be derived without this assumption, though it yields tighter bounds.
\begin{assumption}[Class Convexity]\label{assump-convexity}
    The class $\mathcal{H}$ is convex: given any two hypotheses $h_{1}, h_{2} \in \mathcal{H}$ and any $\theta \in (0,1)$, we have 
    $\theta h_{1} + (1-\theta)h_{2}\in \mathcal{H}.$
\end{assumption}
Note that polynomial regression functions and majority votes over a set of basis functions are examples that satisfy Assumption \ref{assump-convexity}. In contrast, a neural network class with a fixed architecture is generally not closed under convex combinations. Nevertheless, because the Rademacher complexity of a class matches that of its convex hull, the convex hull of a neural network class can be considered instead, which is convex \citep{kalan2025transfer,rigollet2011neyman}

Next, we need a notion of distance between the source and target that translates the performance of a function $h \in \mathcal{H}$ on the source to its performance on the target. We first introduce a general notion of distance and then derive bounds on the generalization error in terms of this notion. This formulation is broad enough to also yield bounds expressed through the commonly used notion of transfer exponent \citep{hanneke2019value,kalantight}.

Let us define $\mathcal{H}^*_{S,T,\alpha}
:= \argmin\limits_{h \in \mathcal{H}_{S}(\alpha)\cap\mathcal{H}_{T}(\alpha)} 
R_{\varphi,\mu_{1,S}}(h),$ which denotes the set of solutions minimizing the source $\varphi$-Type-II risk within the intersection of $\alpha$ feasible sets in the source and target. Among these solutions, let $h^*_{S,T,\alpha}
:= \argmax_{h\in \mathcal{H}^*_{S,T,\alpha}} 
R_{\varphi,\mu_{1,T}}(h)$ denote the pivoting function serving as the reference for comparing errors across domains in the definition of transfer distance. We then define the excess error of any hypothesis $h\in\mathcal{H}$ w.r.t. this pivoting function, for $D\in\{S,T\}$, as  
\begin{equation}
\mathcal{E}_{1,D}(h \mid h^*_{S,T,\alpha})
:= R_{\varphi,\mu_{1,D}}(h)-R_{\varphi,\mu_{1,D}}(h^*_{S,T,\alpha}).
\end{equation}
Using this, we define the transfer modulus, which translates the performance of a function $h\in\mathcal{H}$ from source to target: 
\begin{definition}[Transfer Modulus]\label{def_Transfer_Modulus}
For \(\varepsilon \ge 0\), the transfer moduli for $\varphi$-Type-I and $\varphi$-Type-II errors are defined as follows
\[
\phi_{1}^{S\to T}(\varepsilon)
:= \sup\Big\{
\mathcal{E}_{1,T}(h \mid h^*_{S,T,\alpha})
:\; h\in\mathcal{H},\ \mathcal{E}_{1,S}(h \mid h^*_{S,T,\alpha})\le \varepsilon
\Big\},
\]
\[
\phi_{0}^{S\to T}(\varepsilon)
:= \sup\Big\{
R_{\varphi,\mu_{0,T}}(h) \,:\; h \in \mathcal{H},\ R_{\varphi,\mu_{0,S}}(h) \le \varepsilon
\Big\}.
\]
\end{definition}

Now, equipped with these definitions, we state the following theorem, which provides upper bounds on the generalization error of the proposed transfer learning procedure in Section \ref{section-TL-Procedure}.
\begin{figure}[htbp] 
    \centering
\includegraphics[width=0.7\textwidth, height=.19\textwidth]{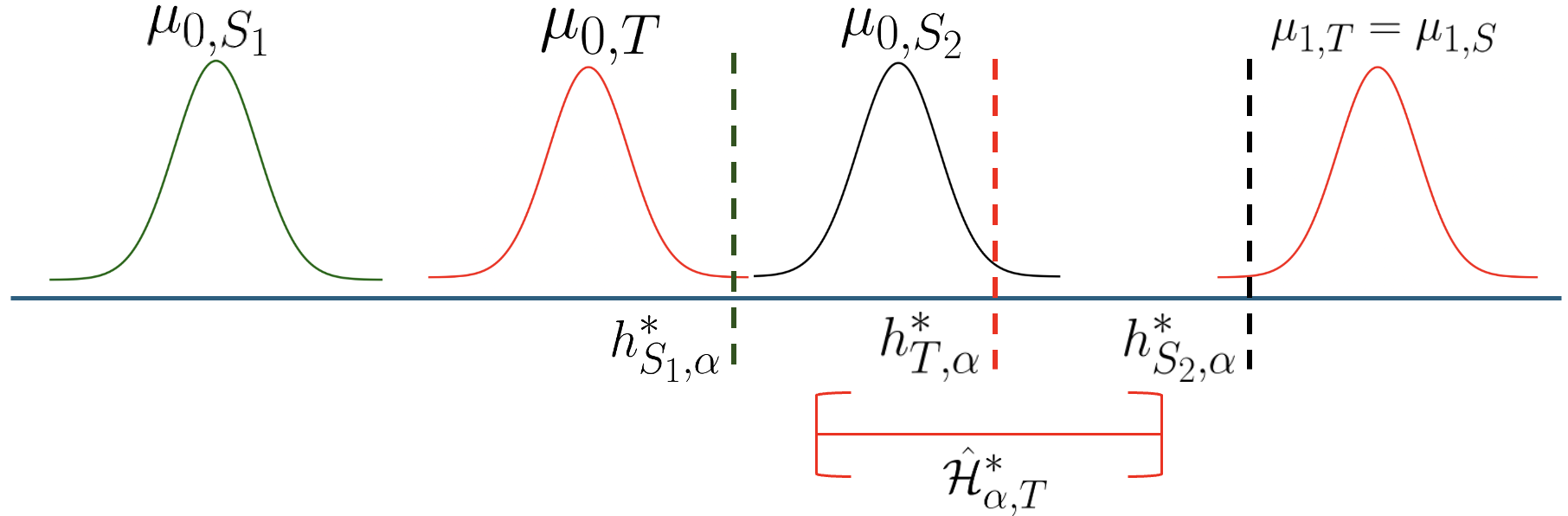} 
    \caption{In this figure, we consider one target and two sources that all share the distribution $\mu_1$, while $\mu_0$ differs across them. All distributions are Gaussian with the same variance but different means. The optimal NP classifiers are denoted by $h^*_{T,\alpha}$, $h^*_{S_1,\alpha}$, and $h^*_{S_2,\alpha}$. Moreover, we assume that there are sufficiently many target samples such that $\hat{\mathcal{H}}^*_{\alpha,T}$ does not intersect with $h^*_{S_1,\alpha}$ or $h^*_{S_2,\alpha}$, which implies that $\alpha_{S_2}>\alpha_{S_1}=\alpha$}
    \label{fig:alpha}
\end{figure}
\vspace{-1mm}
\begin{theorem}\label{theorem_bound}
Suppose that the hypothesis class $\mathcal{H}$ satisfies Assumptions~\ref{assump_rademacher} and \ref{assump-convexity}. Moreover, let $\delta > 0$ and  $\epsilon_{i,D} = \frac{\tilde{C}}{\sqrt{n_{i,D}}}$ for $i \in \{0,1\}$ and $D \in \{S,T\}$, where $
\tilde{C} = 8B_{\mathcal{H}}L + 2C\sqrt{2\log\!\left(\tfrac{2}{\delta}\right)}.$ Assume further that $n_{0,T} \geq n_{1,T}$, that $\mathcal{H}_{T}(\alpha/8)$ is nonempty, and that $n_{0,T}$ is large enough so that $\epsilon_{0,T}\leq \frac{7\alpha}{16}$.  Furthermore, let $\varphi$ be a convex surrogate loss function, and let $\hat{h}$ denote the hypothesis returned by the procedure in Section~\ref{section-TL-Procedure}. Then, with probability at least $1 - 4\delta$, we have:
\begin{align*}
    &\mathcal{E}_{1,T}(\hat{h})
    \leq c\cdot\min\left\{\epsilon_{1,T},\; R_{\varphi,\mu_{1,T}}(h^*_{S,T,\alpha}) 
    - R_{\varphi,\mu_{1,T}}(h^*_{T,\alpha}) + \phi^{S\to T}_{1}(4\epsilon_{1,S})\right\} \\
    &R_{\varphi,\mu_{0,T}}(\hat{h})
    \leq 
    \begin{cases}
        \min\{\alpha+2\epsilon_{0,T},\, \phi^{S\to T}_{0}(\alpha+2\epsilon_{0,S})\}, & \text{if } \alpha \geq \alpha_S, \\[6pt]
        \min\{\alpha+2\epsilon_{0,T},\, \phi^{S\to T}_{0}(\hat{\alpha}_S + 2\epsilon_{0,S})\}, & \text{if } \alpha < \alpha_S.
    \end{cases}
\end{align*}
where $c$ is a universal constant and   $\hat{\alpha}_S$ is the empirical threshold defined in~\eqref{empirical-threshold}. In particular, if $\alpha < \alpha_S$, then $\hat{\alpha}_S \leq \alpha_S$.
\end{theorem}
Note that the assumption $n_{0,T} \geq n_{1,T}$ typically holds in practice, since in the Neyman--Pearson setting class~1 usually corresponds to the rare class, in contrast to class~0. Furthermore, without this assumption, only the bound for the $\varphi$-Type-II error would change, where the term $c \cdot \epsilon_{1,T}$ would be replaced by $c \cdot (\epsilon_{0,T} + \epsilon_{1,T})$.
\begin{remark}
In the special case where $\mu_{0,S}=\mu_{0,T}$, Theorem~\ref{theorem_bound} recovers the bounds in \cite{kalan2025transfer,kalantight}. In this setting, we have $\alpha_S \leq \alpha$ by \eqref{alpha_s_definition}, and $\phi^{S\to T}_{0}$ reduces to the identity function. Consequently, the bound for the $\varphi$-Type-I error simplifies to $\alpha + \epsilon_0$, where $\epsilon_0=\epsilon_{0,S}=\epsilon_{0,T}$. Regarding the $\varphi$-Type-II excess error bound, in this case we have $h^*_{S,T,\alpha}=h^*_{S,\alpha}$. Furthermore, in the case where source and target distributions are identical, i.e., $\mu_{i,S}=\mu_{i,T}$ for $i\in\{0,1\}$, as $n_{0,S}, n_{1,S}\to\infty$, the $\varphi$-Type-II excess error converges to $0$, while the upper bound on the $\varphi$-Type-I error converges to $\alpha$.
\end{remark}
\begin{remark}
Theorem~\ref{theorem_bound} guarantees the adaptivity of the proposed transfer learning procedure without requiring any prior knowledge of the relatedness between source and target. It recovers the performance of using only target data, thereby avoiding negative transfer when the source is uninformative. Moreover, whenever the source is related and informative, the procedure leverages it to achieve better performance compared to using only the target, in terms of both $\varphi$-Type-I and $\varphi$-Type-II errors.
\end{remark}
\begin{remark}
In Theorem \ref{theorem_bound}, the source effect is captured by the functions $\phi_{1}^{S \to T}$ and $\phi_{0}^{S \to T}$. In contrast, \cite{kalan2025transfer,kalantight} use the notion of transfer exponent \citep{hanneke2019value} to characterize the source effect. In Appendix \ref{appendix_transfer_exponent}, we show that Theorem \ref{theorem_bound} can also be expressed in terms of the transfer exponent.
\end{remark}
\begin{remark}
Figure \ref{fig:alpha} illustrates the effect of $\alpha_S$ on reducing the target Type-I error. In this example, $\mathcal{H}=\{\mathds{1}\{x\geq t\}: t\in \mathbb{R}\}$ and we consider one target and two sources that all share the distribution $\mu_1$, while $\mu_0$ differs across them. All distributions are Gaussian with the same variance but different means. For source-1, $\alpha_{S_1}=\alpha$, which implies that it does not shrink $\hat{\mathcal{H}}^*_{\alpha,T}$ since $\hat{\mathcal{H}}_{S_1}(\hat{\alpha}_S)\supset\hat{\mathcal{H}}^*_{\alpha,T}$, and thus source-$1$ cannot reduce the target Type-I error. In contrast, for source-2, $\alpha_{S_2}>\alpha$, and $\hat{\mathcal{H}}'=\hat{\mathcal{H}}_S(\hat{\alpha}_{S_2})\cap\hat{\mathcal{H}}^*_{\alpha,T}$ yields an improvement in the target Type-I error.
\end{remark}





\subsection{Computational Guarantee for the Transfer Learning Procedure}
In this subsection, we provide a computational guarantee for the transfer learning procedure introduced in Section~\ref{section-TL-Procedure} by reformulating it as an optimization procedure. The optimization procedure consists of two stages: an auxiliary algorithm for computing $\hat{\alpha}_S$ in \eqref{empirical-threshold}, and a main algorithm for solving the final model $\hat{h}$ in \eqref{procedure_intersection}. For this optimization procedure, we consider a parameterized hypothesis class $\mathcal{H} \doteq \{h_{\theta}: \theta\in \mathbb{R}^d, ||\theta||\leq B\}$. 

\paragraph{Notations.} $\norm{\cdot}$ denotes the Euclidean ($\ell_2$) norm. In this subsection,  $\hat{R}_{\varphi,\mu}(\theta)$ is a short hand for $\hat{R}_{\varphi,\mu}(h_\theta)$. The notation $f(\theta;z)$ denotes the evaluation of the function $f$ at $z$, parameterized by $\theta$.

We make the following assumption on the loss function $\varphi$ and the hypothesis class.
 \begin{assumption}[Convexity, Bounded and Lipschitz Gradient]\label{ass:opt}
We assume that $\varphi\!\left((2y-1) h_{\theta}(x)\right)$ is convex in $\theta$ for any $x \in \mathcal{X}$ and $y \in \{0,1\}$. In addition, the gradient is bounded, i.e.,$
\|\nabla_{\theta} \varphi((2y-1) h_{\theta}(x))\| \leq G$, for all $x \in \mathcal{X}$ and $\theta$ with $\|\theta\|\leq B$, and the gradient is Lipschitz continuous, i.e., 
\[
\|\nabla_{\theta} \varphi((2y-1) h_{\theta}(x)) - \nabla_{\theta} \varphi((2y-1) h_{\theta'}(x))\| \leq H\|\theta - \theta'\|,
\] for all $x \in \mathcal{X}$, $y \in \{0,1\}$, and $\theta, \theta'$ with $\|\theta\|,\|\theta'\|\leq B$.
\end{assumption}
The above assumption is standard and holds for common loss functions and hypothesis classes, such as the logistic loss, linear classifiers, and majority votes over a set of basis functions. 

\textbf{First Stage: Computing $\bm{\hat{\alpha}}_S$.} 
Define 
$g_{0,T}(\theta) \doteq \hat{R}_{\varphi,\mu_{0,T}}(\theta) - \alpha - \epsilon_{0,T}$, 
$g_{0,S}(\theta,\alpha') \doteq \hat{R}_{\varphi,\mu_{0,S}}(\theta) - \alpha' - \epsilon_{0,S}$, 
and 
$g_{1,T}(\theta) \doteq \hat{R}_{\varphi,\mu_{1,T}}(\theta) - \hat{R}_{\varphi,\mu_{1,T}}(\hat{\theta}_{T,\alpha-\epsilon_{0,T}}^*) - \epsilon_{1,T}$. By definition, $\hat{\alpha}_S$ is the smallest $\alpha' \geq \alpha$ such that the intersection of the sets $\cbr{\theta: g_{0,T}(\theta) \leq 0}$, $\cbr{\theta:g_{0,S}(\theta,\alpha')\leq 0}$ and $\cbr{\theta: g_{1,T}(\theta) \leq 0}$ is non-empty. This can be formulated as the following  optimization problem: 
\begin{align} 
\min_{\alpha',\theta} \ \alpha' 
\quad \text{s.t.} \quad 
\alpha' \geq \alpha,\; g'(\theta,\alpha')\leq 0. \label{eq:objective alpha'}
\end{align}
where $g'(\theta,\alpha')\doteq \max\cbr{ g_{0,T}(\theta), g_{0,S}(\theta,\alpha'), g_{1,T}(\theta) }$.
  Problem \eqref{eq:objective alpha'} is equivalent to the following minimax problem, in the sense that the primal solution of \eqref{eq:Lagrangian1} coincides with the solution of \eqref{eq:objective alpha'}~\cite[Section~5.2.3]{boyd2004convex}:
\begin{align}
    \min_{\alpha'\geq \alpha,\theta} \max_{\lambda \geq 0} \alpha' + \lambda g'(\theta,\alpha').  \label{eq:Lagrangian1}
\end{align}
\vspace{-4mm}

\begin{algorithm2e}
\SetKwInOut{Input}{Input}
\SetKwInOut{Output}{Output}
\SetKwFunction{CPSolver}{CP\textnormal{-}Solver}
\SetKw{KwAnd}{and}
\SetInd{1em}{1em} 
\DontPrintSemicolon  
     \caption{ $\protect\CPSolver(f,  g, \xi, \epsilon)$} 
    	\label{algorithm:CP}
    \textbf{Input:}  objective  $f(\theta)\doteq  \frac{1}{|\mathcal{Z}_f|}\sum_{\zeta\in \mathcal{Z}_f}f(\theta;\zeta)$, constraint $ g(\theta)\doteq \frac{1}{| \mathcal{Z}_g|}\sum_{\zeta\in \mathcal{Z}_g} g(\theta;\zeta)$,  slackness $\delta$, error tolerance $\epsilon$, failure probability $\delta$.\\
    \textbf{Initialize:} $(\theta_0,\lambda_0) = (0,0)$, $\eta = \frac{c_{\eta}}{\sqrt{N(\epsilon)}}$,   $G  = \sup_{\theta,\zeta_f,\zeta_g}\max\{\norm{\nabla f(\theta;\zeta_f)},\norm{\nabla  g(\theta;\zeta_g)} \}$, $\gamma =  G^2 \eta$.
    
        	\For{$t = 0,...,N(\epsilon)-1$}
{
    Sample $\zeta_{f,t}$, $\zeta_{g,t}$ uniformly from $\mathcal{Z}_f$ and $\mathcal{Z}_g$ respectively and uniformly\\
    $\theta_{t+1} =  \theta_t - \eta ( \nabla f(  \theta_t;\zeta_{f,t} ) ) +  \lambda_t  \nabla g(\theta_t;\zeta_{g,t} ) )$\\
    $\theta_{t+1} = \theta_{t+1}/\max\cbr{1,\norm{\theta_{t+1}}/B}$\\
    $\lambda_{t+1} = [(1-\gamma \eta )\lambda_t + \eta    g(\theta_t; \zeta_{g,t}) ]_+$

}
\textbf{Output:} $\hat\theta =$ projection of $\frac{1}{N}\sum_{t=1}^N \theta_t$ onto the set $\cbr{\theta:  g(\theta)\leq -\xi}$
\end{algorithm2e}

To solve \eqref{eq:Lagrangian1}, we need an approximation of $\hat{\theta}_{T,\alpha-\epsilon_{0,T}}^*$. This can be obtained by solving another  optimization problem: $\min_{\theta} \hat{R}_{\varphi,\mu_{1,T}}(\theta) \ \text{s.t.}\ g_{0,T}^-(\theta) \doteq \hat{R}_{\varphi,\mu_{0,T}}(\theta) - \alpha+ \epsilon_{0,T} \leq 0$, which is equivalent to  $\min_{\theta}\max_{\lambda\geq 0} \hat{R}_{\varphi,\mu_{1,T}}(\theta) +\lambda g_{0,T}^-(\theta)  $. Following~\citep{mahdavi2012stochastic}, we solve this problem using a Stochastic Gradient Descent–Ascent (SGDA) method with a projection step, as described in Algorithm~\ref{algorithm:CP} $\CPSolver$.  This algorithm solves convex programs which takes objective and constraint (possibly inexact) and returns an $\epsilon$ accurate solution. Since we only have access to inexact constraint $g$, the final projection step is onto $\cbr{\theta: g(\theta) \leq -\xi}$ to allow some slackness.

Once an approximation $\hat{\theta}_{T,\alpha-\epsilon_{0,T}}$ is obtained, we return to \eqref{eq:Lagrangian1}, replacing $g_{1,T}(\theta)$ with $\hat g_{1,T}(\theta )\doteq \hat{R}_{\varphi,\mu_{1,T}}(\theta) - \hat{R}_{\varphi,\mu_{1,T}}(\hat{\theta}_{T,\alpha-\epsilon_{0,T}}) - 6\epsilon_{1,T}$, and again call Algorithm~\ref{algorithm:CP} to compute an approximation of $\hat{\alpha}_S$. The following key quantity will appear in the analysis of Algorithm~\ref{algorithm:CP}.

\begin{definition} \label{def:constraint lower bound} We  define $r(g(\theta))  \doteq  \inf  \cbr{ \| 
 v\|:  v \in \partial g(\theta), g(\theta) = 0, \theta \in \operatorname{dom}(g)}$.
\end{definition}

The above notion, introduced by~\cite{mahdavi2012stochastic}, measures the magnitude of the constraint gradient along the boundary of the constraint set. It is used in the convergence analysis to control the distance between the returned solution and the solution prior to the final projection step.
\vspace{-2mm}

\begin{algorithm2e}[t]
\small                 
\SetKwInOut{Input}{Input}
\SetKwInOut{Output}{Output}
\SetKwFunction{CPSolver}{CP\textnormal{-}Solver}
\SetKw{KwAnd}{and}
\SetInd{1em}{1em} 
\DontPrintSemicolon  
\caption{NP-Transfer-Learning}
\label{algorithm: hat theta}


\BlankLine
\textbf{Initialize:}
$G_\lambda \gets 2C+\max\!\left\{\epsilon_{0,S},\epsilon_{0,T},\epsilon_{1,S},\epsilon_{1,T}\right\}$\;
$r_T \gets r\!\left(g_{0,T}^-(\theta)\right)$, $r_{\hat{\alpha}_S} \gets r\!\left(g'(\theta,\alpha')\right)$, $r_S' \gets r\!\left(g'_S(\theta)\right)$,$r_T' \gets r\!\left(g'_T(\theta)\right)$, $r_{S,T} \gets r\!\left(g_{S,T}(\theta)\right)$

\BlankLine
\textbf{Define the slackness function}: $
\xi(\epsilon, r)\;\doteq\;
\min\!\left\{
  \frac{1}{4H},
  \left(G + G_\lambda\sqrt{\log\frac{1}{\delta}}\right)^{-1}
  \cdot \frac{\epsilon}{2+2H\epsilon}
\right\}\, r.$

\BlankLine
\textbf{Tolerance setup:}\;
$\epsilon_{S,T} \gets \epsilon_{1,S}$,\quad
$\epsilon'_T \gets \min\!\left\{ \xi(\epsilon_{S,T}, r_{S,T}),\ \epsilon_{1,T} \right\}$,\quad
$\epsilon'_S \gets \epsilon_{1,S}$\; 

$\epsilon_{\hat\alpha_S}
  \gets \min\!\left\{
    \xi(\epsilon_{S,T}, r_{S,T}),
    \xi(\epsilon'_S, r_S'),           
    \xi(\epsilon'_T, r_T')           
  \right\}$\;

$\epsilon_{T,\alpha-\epsilon_{0,T}}
  \gets \min\!\left\{
    \xi(\epsilon_{S,T}, r_{S,T}),
    \xi(\epsilon'_S, r_S'),
    \xi(\epsilon'_T, r_T'),
    \xi(\epsilon_{\hat{\alpha}_S}, r_{\hat{\alpha}_S})
  \right\}$\;

\BlankLine
\textbf{Warm-start on $T$:}\;
$\hat{\theta}_{T,\alpha-\epsilon_{0,T}}
   \gets \CPSolver\!\left(
      \hat R_{\varphi,\mu_{1,T}},\ g_{0,T}^-,\ 0,\ \epsilon_{T,\alpha-\epsilon_{0,T}}
   \right)$\;

$\hat g_{1,T}(\theta)
 \gets \hat R_{\varphi,\mu_{1,T}}(\theta)
   - \hat R_{\varphi,\mu_{1,T}}\! (\hat{\theta}_{T,\alpha-\epsilon_{0,T}} )
   - 6\epsilon_{1,T}$,  $\hat g'(\theta,\alpha)
 \gets \max\!\left\{
      g_{0,T}(\theta),\
      g_{0,S}(\theta,\alpha),\
      \hat g_{1,T}(\theta)
 \right\}$\;

\BlankLine
\textbf{Compute $\hat\alpha$:} 
$(\hat{\alpha},\perp) \gets
\CPSolver\!\left(
  \alpha',\ \hat g'(\theta,\alpha'),\ \xi(\epsilon_{\hat\alpha_S}, r_{\hat\alpha_S}),\ \epsilon_{\hat\alpha_S}
\right)$\;

\BlankLine
\textbf{Define the joint constraint}:\
$\hat g_{\hat{\alpha}}(\theta)
 \gets \max\!\left\{
   g_{0,T}(\theta),\ g_{0,S}(\theta,\hat{\alpha}),\ \hat g_{1,T}(\theta)
 \right\}$\;

\BlankLine
\textbf{Solve $T$ and $S$ subproblems:}\;
$\hat{\theta}'_{T}
 \gets \CPSolver\!\left(
   \hat R_{\varphi,\mu_{1,T}},\ \hat g_{\hat{\alpha}},\ \xi(\epsilon'_T, r_T'),\ \epsilon'_T
 \right)$ ; $\hat{\theta}'_{S}
 \gets \CPSolver\!\left(
   \hat R_{\varphi,\mu_{1,S}},\ \hat g_{\hat{\alpha}},\ \xi(\epsilon'_S, r_S'),\ \epsilon'_S
 \right)$\;

\BlankLine
\textbf{Define the final constraint:}\;
$\hat g'_T(\theta)
 \gets \hat R_{\varphi,\mu_{1,T}}(\theta)
   - \hat R_{\varphi,\mu_{1,T}}(\hat{\theta}'_{T})
   - 2\epsilon_{1,T}$ ; $\hat g_{S,T}(\theta)
 \gets \max\!\left\{
   \hat g_{\hat{\alpha}}(\theta),\ \hat g'_T(\theta)
 \right\}$\;

\BlankLine
\textbf{Final solve $\hat{h}$:}\;
$\tilde{\theta}
 \gets \CPSolver\!\left(
   \hat R_{\varphi,\mu_{1,S}},\ \hat g_{S,T},\ \xi(\epsilon_{S,T}, r_{S,T}),\ \epsilon_{S,T}
 \right)$\;

\BlankLine
\textbf{Output:} 
\textbf{If }{$\hat R_{\varphi,\mu_{1,S}}(\tilde\theta) - \hat R_{\varphi,\mu_{1,S}}(\hat{\theta}'_S) > 2\epsilon_{1,S}$: } $\hat\theta \gets \hat\theta'_T$ \textbf{else: }{
  $\hat\theta \gets \tilde\theta$
}

\end{algorithm2e}
\paragraph{Second Stage: Solving the Final Model $\hat h$.}
We reduce the final step of the procedure in Section~\ref{section-TL-Procedure}—which checks whether $\hat{\mathcal{H}}'_{1,S}$ and $\hat{\mathcal{H}}'_{1,T}$ intersect, as defined in \eqref{procedure_intersection}—to a constrained optimization problem.
We define 
$g_{\hat{\alpha}_S}(\theta) \doteq \max\cbr{g_{0,T}(\theta), g_{0,S}(\theta,\hat{\alpha}_S), g_{1,T}(\theta) }.$ 
The set $\{\theta: g_{\hat{\alpha}_S} (\theta)\leq 0\}$ coincides with the hypothesis class $\hat{\mathcal{H}}'$ defined in Section~\ref{section-TL-Procedure}.
We also define 
\begin{align}
    g'_D (\theta) \doteq \hat{R}_{\varphi,\mu_{1,D}}(\theta) - \hat{R}_{\varphi,\mu_{1,D}}(\hat{\theta}_{D,\hat\alpha_S}^*) - 2\epsilon_{1,D} \label{eq:gD prime}
\end{align}
where $\hat{\theta}_{D,\hat\alpha_S}^* = \arg\min_{\theta}\hat{R}_{\varphi,\mu_{1,D}}( \theta) \ \text{s.t.} \ g_{\hat\alpha_S}(\theta) \leq 0$ for $D\in\cbr{S,T}$. The set $\{\theta:  g_{\hat{\alpha}_S}(\theta)\leq 0, g'_{D}(\theta)\leq 0 \}$ coincides with the hypothesis class $\hat{\mathcal{H}}'_{1,D}$ as in~\eqref{procedure_intersection}. Then, the second stage of the algorithm is to solve the following problem:
\begin{align}
    \min_{\theta} \ \hat{R}_{\varphi,\mu_{1,S}}(\theta) \quad \text{s.t.} \quad g_{S,T}(\theta) \leq 0 \label{eq:second stage}
\end{align}
where $g_{S,T}(\theta) \doteq \max\cbr{g_{\hat{\alpha}_S}(\theta) ,g_{T}'(\theta)  }$. To approximate $\hat{\theta}_{D,\hat\alpha_S}^*$ in~\eqref{eq:gD prime} for $D\in \cbr{S,T}$, we again call Algorithm~\ref{algorithm:CP} to solve $ \min_{ \theta:g_{\hat{\alpha}_S}(\theta) \leq 0}  \ \hat{R}_{\varphi, \mu_{1,D}}(\theta)$
to a certain accuracy prescribed in Algorithm~\ref{algorithm: hat theta}. The whole algorithm (first stage + second stage) is presented in Algorithm~\ref{algorithm: hat theta}.

\begin{theorem}\label{thm:opt rate}
Let Assumption~\ref{ass:opt} hold. With properly chosen $c_{\eta}$ and $N(\epsilon)$ in Algorithm~\ref{algorithm:CP}, then the
Algorithm~\ref{algorithm: hat theta} outputs $\hat \theta$ such that, with probability at least $1-5\delta$, the hypothesis $h_{\hat{\theta}}$ satisfies the statistical guarantee in Theorem~\ref{theorem_bound}, and the number of stochastic gradient evaluations is bounded by 
\begin{align*}
   c\cdot \pare{   G  +    G_{\lambda}  \sqrt{ \log\frac{1}{\delta}}  }^3\pare{  \frac{   G  G_{\lambda}   \sqrt{ \log\frac{1}{\delta}} +     \frac{   G}{c_\eta}  }{ r_{\min}^2 }      +  \frac{B^2+\frac{G^2}{r_{\min}^2} }{c_{\eta}}  +\frac{G}{r_{\min}} G_{\lambda}   \sqrt{ \log\frac{1}{\delta}}  }^2 \cdot \frac{1}{\epsilon_{\min}^2}
\end{align*}
for some  $c>0$,
where $\epsilon_{\min} \doteq  \min\cbr{\epsilon_{1,S}, \epsilon_{1,T},  \xi(\epsilon_{S,T}, r_{S,T}), \xi(\epsilon'_{S}, r'_{S}), \xi(\epsilon'_{T}, r'_{T}), \xi(\epsilon_{\hat{\alpha}_S},r_{\hat{\alpha}_S} )},$ and $r_{\min}\doteq \min\cbr{r_T,  r_{\hat\alpha_S}, r'_{T}, r'_{S}, r_{S,T} }$ where $\xi$ is defined in Algorithm~\ref{algorithm: hat theta}.
\end{theorem}

\begin{remark}[Computational Complexity]
    The proposed algorithm outputs a model that satisfies the learning guarantee as in Theorem~\ref{theorem_bound}, and the total number of (stochastic) gradient evaluations scales as $\min\{\tfrac{1}{\epsilon_{1,S}}, \tfrac{1}{\epsilon_{1,T}}\}$. Roughly speaking, we call $\CPSolver$ procedures multiple times to solve different convex programs to different accuracies, and the final gradient complexity is dominated by the worst one of them.  
\end{remark}
\vspace{-3mm}
\section{Experiments}
\vspace{-2mm}
In this section, we implement the proposed algorithm on two climate datasets: \citet{yu2023climsim}, with 124 features, and \citet{nasa_power_api}, with 6 features, for heavy rain detection, where source and target correspond to different locations. We use a multilayer perceptron with two hidden layers and ReLU activation. In these experiments, the number of target samples is fixed at $n_{0,T} = n_{1,T} = 40$, while the number of source samples $n_{0,S} = n_{1,S}$ varies from 50 to 950, with $\alpha=0.1$. Each case is run for 10 trials, and the trained model is evaluated on a target test set of size $n_{0,T} = n_{1,T} = 1700$.

In Figure \ref{fig1},  the proposed Transfer Learning Algorithm (TLA) keeps the Type-I error close to the threshold, whereas the Only Target method exceeds it more and also performs poorly on the Type-II error. In this figure, the source is informative, and TLA leverages it to outperform Only Target method without requiring prior knowledge of the relatedness between source and target.

In Figure \ref{fig2}, the source is not informative, and the Only Source method suffers from a high Type-II error. TLA, however, adaptively avoids negative transfer and achieves performance comparable to the Only Target method, corroborating the result in Theorem \ref{theorem_bound}. Hence, when the source is informative, our algorithm adaptively exploits it, and when it is not, it avoids negative transfer, even in the presence of potential shifts in both class-$0$ and class-$1$ distributions.


\begin{figure}[htbp]
\centering
\begin{tikzpicture}
  \begin{groupplot}[
    group style={group size=2 by 1, horizontal sep=8em},
    myaxis
  ]

  \nextgroupplot[
    xlabel={$n_{0,S}=n_{1,S}$},
    ylabel={Target Type-II error},
    xmin=50, xmax=950,
    xtick={50,250,450,650,850},      
    minor xtick={150,350,550,750,950}, 
    ytick={0.02,0.05,0.1,.15,0.2,0.25,0.3}
  ]
    \addplot+[blue, thick, mark=*] coordinates {
      (50,0.0796) (150,0.0593) (250,0.0607) (350,0.0344) (450,0.0356)
      (550,0.0365) (650,0.0439) (750,0.0311) (850,0.0237) (950,0.0242)};
    \addplot+[red, thick, mark=x] coordinates {
      (50,0.1424) (150,0.2225) (250,0.1825) (350,0.1181) (450,0.1458)
      (550,0.2409) (650,0.0939) (750,0.1576) (850,0.1215) (950,0.1323)};
    \addplot+[black, thick, mark=triangle*] coordinates {
      (50,0.1084) (150,0.1748) (250,0.0724) (350,0.1547) (450,0.0833)
      (550,0.0822) (650,0.0473) (750,0.0541) (850,0.0455) (950,0.0402)};
    \legend{TLA,Only Target,Only Source}

  \nextgroupplot[
    xlabel={$n_{0,S}=n_{1,S}$},
    ylabel={Target Type-I error},
    xmin=50, xmax=950,
    xtick={50,250,450,650,850},
    minor xtick={150,350,550,750,950},
    ymin=0.08, ymax=0.18,                
  ytick={0.08,0.10,0.12,0.14,0.17},
  scaled y ticks=false,
  yticklabel style={
    /pgf/number format/.cd,
    fixed, precision=2, zerofill
  }
  ]
    \addplot+[blue, thick, mark=*] coordinates {
      (50,0.1172) (150,0.1112) (250,0.1157) (350,0.1192) (450,0.1181)
      (550,0.1114) (650,0.1086) (750,0.1179) (850,0.1142) (950,0.1119)};
    \addplot+[red, thick, mark=x] coordinates {
      (50,0.0982) (150,0.1042) (250,0.1491) (350,0.1110) (450,0.1036)
      (550,0.1545) (650,0.1168) (750,0.1124) (850,0.1328) (950,0.0991)};
    \addplot+[black, thick, mark=triangle*] coordinates {
      (50,0.1286) (150,0.0930) (250,0.0998) (350,0.0893) (450,0.0906)
      (550,0.0893) (650,0.0970) (750,0.0912) (850,0.0932) (950,0.0934)};
   \legend{TLA,Only Target,Only Source}

  \end{groupplot}
\end{tikzpicture}

\caption{\footnotesize The performance of our algorithm (TLA), along with two baselines---using only source data and only target data---on Climate data~\citep{yu2023climsim}, is evaluated under a Type-I error threshold of $\alpha = 0.1$. In this experiment, we fix $n_{0,T} = n_{1,T} = 40$ and vary the number of source samples $n_{0,S} = n_{1,S}$.
}\label{fig1}
\end{figure}
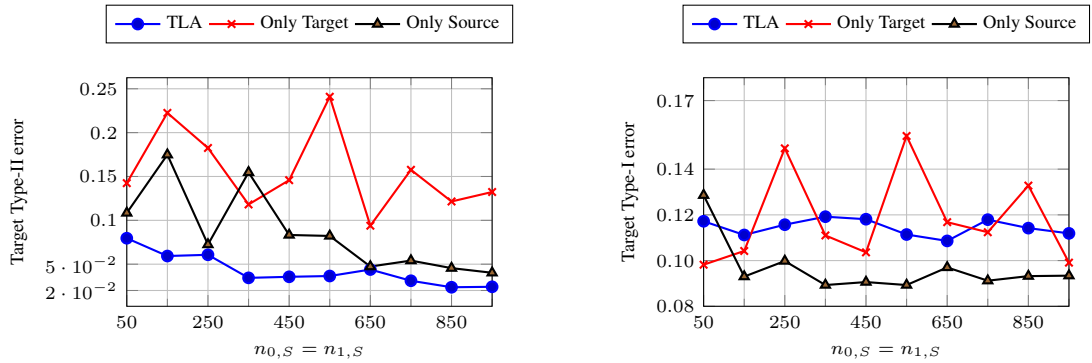

\begin{figure}[htbp]
\centering
\begin{tikzpicture}
  \begin{groupplot}[
    group style={group size=2 by 1, horizontal sep=8em},
    myaxis
  ]

  \nextgroupplot[
    xlabel={$n_{0,S}=n_{1,S}$},
    ylabel={Target Type-II error},
    xmin=50, xmax=950,
    xtick={50,250,450,650,850},      
    minor xtick={150,350,550,750,950}, 
    ytick={0.2,0.3,0.5,0.7,0.9}
  ]
    \addplot+[blue, thick, mark=*] coordinates {
      (100,0.1726470582) (200,0.2145882368) (300,0.2041176483) (400,0.2094117627)
      (500,0.2348823532) (600,0.2195882365) (700,0.2949411780) (800,0.2454705864)
      (900,0.2600588217) (1000,0.2701176465)
    };
    \addplot+[red, thick, mark=x] coordinates {
      (100,0.1923529424) (200,0.1782352969) (300,0.1867058821) (400,0.2480588235)
      (500,0.2000588238) (600,0.2212352961) (700,0.2614705905) (800,0.2521176457)
      (900,0.1719411764) (1000,0.1930000007)
    };
    \addplot+[black, thick, mark=triangle*] coordinates {
      (100,0.7145882338) (200,0.8737058878) (300,0.7778235316) (400,0.8012352914)
      (500,0.7665882409) (600,0.7140588254) (700,0.7813529313) (800,0.7477058858)
      (900,0.8648235261) (1000,0.8180588305)
    };
    \legend{TLA,Only Target,Only Source}

  \nextgroupplot[
    xlabel={$n_{0,S}=n_{1,S}$},
    ylabel={Target Type-I error},
    xmin=50, xmax=950,
    xtick={50,250,450,650,850},      
    minor xtick={150,350,550,750,950}, 
    ytick={0.02,0.05,0.1,0.15,0.2}
  ]
    \addplot+[blue, thick, mark=*] coordinates {
      (100,0.1885294124) (200,0.1791764714) (300,0.1820000000) (400,0.2030588269)
      (500,0.1734117642) (600,0.1762352943) (700,0.1592352943) (800,0.1657647073)
      (900,0.1622352973) (1000,0.1587058805)
    };
    \addplot+[red, thick, mark=x] coordinates {
      (100,0.1850000000) (200,0.2217647091) (300,0.1635882340) (400,0.1641176485)
      (500,0.1979411796) (600,0.1575294130) (700,0.1672352973) (800,0.1657647055)
      (900,0.2105882332) (1000,0.2085294098)
    };
    \addplot+[black, thick, mark=triangle*] coordinates {
      (100,0.1323529424) (200,0.0485882354) (300,0.0367647063) (400,0.0512941181)
      (500,0.0574705880) (600,0.0627058833) (700,0.0252352939) (800,0.0697058824)
      (900,0.0114117647) (1000,0.0138235293)
    };
    \legend{TLA,Only Target,Only Source}

  \end{groupplot}
\end{tikzpicture}
\caption{The performance of our algorithm (TLA), along with two baselines---using only source data and only target data---on Climate data~\citep{nasa_power_api}, is evaluated under a Type-I error threshold of $\alpha = 0.1$. In this experiment, we fix $n_{0,T} = n_{1,T} = 40$ and vary the number of source samples $n_{0,S} = n_{1,S}$.}\label{fig2}
\end{figure}
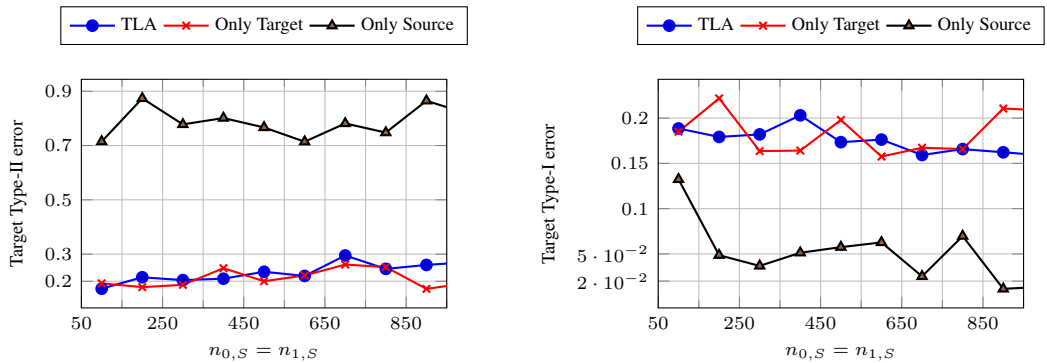




\clearpage
\bibliographystyle{plainnat}
\bibliography{reference}

\newpage
\appendix
\section{Expressing Theorem \ref{theorem_bound} in Terms of Transfer Exponent}\label{appendix_transfer_exponent}
Theorem \ref{theorem_bound} uses the functions $\phi_{1}^{S \to T}$ and $\phi_{0}^{S \to T}$ to capture the source effect in the generaliztion bound. This formulation is broad enough to yield bounds expressed through the notion of transfer exponent \citep{hanneke2019value,kalantight}.

Regarding the $\varphi$-Type-II excess error, assume that there exist constants $C_1 > 0$ and $\rho_1 > 0$ such that for all $h \in \mathcal{H}$,
\[
\Bigl[\mathcal{E}_{1,S}(h \mid h^*_{S,T,\alpha})\Bigr]_{+} \;\geq\; C_1 \, \Bigl[\mathcal{E}_{1,T}(h \mid h^*_{S,T,\alpha})\Bigr]_{+}^{\rho_1}.
\]
Then, by Definition \ref{def_Transfer_Modulus}, for any $\varepsilon \geq 0$ we have $\phi_{1}^{S \to T}(\varepsilon) \;\leq\; C_1^{1/\rho_1}\, \varepsilon^{1/\rho_1}$, and consequently the term $\phi_{1}^{S \to T}(c \cdot \epsilon_{1,S})$ in Theorem \ref{theorem_bound} can be replaced with $c \cdot \epsilon_{1,S}^{1/\rho_1}$ for some numerical constant $c$. Regarding the $\varphi$-Type-I error, first we define the minimal source level that guarantees $\alpha$ level on target as
\[
\alpha^{*}_{S \to T} \;:=\; \inf \bigl\{ \alpha' \in [0,1] : \; \phi_{0}^{S \to T}(\alpha') \leq \alpha \bigr\}.
\]
Assume that there exist constants $C_0>0$ and $\rho_0>0$ such that for all $h\in \mathcal{H}$
\[
\bigl[ R_{\varphi,\mu_{0,S}}(h) - \alpha^{*}_{S \to T}(\alpha) \bigr]_{+} \;\geq\; 
C_0 \bigl[ R_{\varphi,\mu_{0,T}}(h) - \alpha \bigr]_{+}^{\rho_0}.
\]
Then the $\varphi$-Type-I error bound takes the form
\[
R_{\varphi,\mu_{0,T}}(\hat{h}) \;\leq\;
\begin{cases}
\alpha + \min \left\{ \epsilon_{0,T}, \; c\cdot \Bigl( \bigl[\alpha + \epsilon_{0,S} - \alpha^{*}_{S \to T}(\alpha)\bigr]_{+} \Bigr)^{1/\rho_0} \right\}, & \alpha \geq \alpha_S, \\[1.2em]
\alpha + \min \left\{ \epsilon_{0,T}, \; c\cdot \Bigl( \bigl[\hat{\alpha}_S + \epsilon_{0,S} - \alpha^{*}_{S \to T}(\alpha)\bigr]_{+} \Bigr)^{1/\rho_0} \right\}, & \alpha < \alpha_S.
\end{cases}
\]
In particular, if $\mu_{0,S}=\mu_{0,T}$ then $\alpha^{*}_{S \to T}=\alpha$, and the bound matches that of \cite{kalan2025transfer,kalantight}.
\section{Proof of Theorem \ref{theorem_bound}}
We begin with two lemmas. The first establishes a uniform concentration bound for the difference between empirical and population errors. The second shows that the $\varphi$-Type-II error can be controlled by the deviation in the $\varphi$-Type-I error, provided that both the hypothesis class $\mathcal{H}$ and the loss $\varphi$ are convex.
\begin{lemma}\label{Lemma_Uniform_Concentration}
Let $\delta > 0$ and let $\mathcal{H}$ be a hypothesis class satisfying Assumption \ref{assump_rademacher}. Denote by $\hat{R}_{\varphi,\mu}$ the empirical error computed from $n$ i.i.d.\ samples drawn from a distribution $\mu$, where $\mu$ may be either $\mu_0$ or $\mu_1$. Then, with probability at least $1 - \delta$,  
\[
\sup_{h \in \mathcal{H}} \bigl| R_{\varphi,\mu}(h) - \hat{R}_{\varphi,\mu}(h) \bigr|
\;\leq\; \frac{4 B_{\mathcal{H}} L + C \sqrt{2 \log (2/\delta)}}{\sqrt{n}},
\]
where the constant $C$ is as defined in Definition \ref{def_loss}.
\end{lemma}
\begin{proof}
See the proof of Proposition~1 in \cite{kalan2025transfer}, which relies on symmetrization and McDiarmid's inequality \citep[Chapter~26]{shalev2014understanding}.
\end{proof}
\begin{lemma}\label{Lemma_Type-II-error-control}
Suppose that the hypothesis class $\mathcal{H}$ satisfies Assumption~\ref{assump-convexity}, and that the surrogate loss $\varphi$ in Definition~\ref{def_loss} is convex. Further assume that $\mathcal{H}_T(\alpha/8)$ is nonempty and that $\epsilon_{0,T} \leq \tfrac{7\alpha}{8}$. Then,
\begin{equation}
R_{\varphi,\mu_{1,T}}\!\left(h^*_{T,\alpha-\epsilon_{0,T}}\right)
    - R_{\varphi,\mu_{1,T}}\!\left(h^*_{T,\alpha}\right)
    \;\leq\; c \cdot \epsilon_{0,T},
\end{equation}
for some numerical constant $c$.
\end{lemma}
\begin{proof}
See the proof of Theorem~2 in \cite{kalan2025transfer}, which shows that 
\[
\gamma(\alpha) \;:=\; \inf_{h \in \mathcal{H}_{T}(\alpha)} R_{\varphi,\mu_{1,T}}(h)
\]
is a non-increasing convex function on $[0,1]$.
\end{proof}

Consider the event that Lemma 1 holds simultaneously for the distributions $\mu_{i,D}$ with $i \in \{0,1\}$ and $D \in \{S,T\}$; this event occurs with probability at least $1 - 4\delta$. We divide the proof into two cases: (i) $\alpha \geq \alpha_S$, and (ii) $\alpha < \alpha_S$.

\textbf{Case I:} $\alpha \geq \alpha_S.$ \\[6pt]
First, we show that $\hat{\alpha}_S = \alpha$. From \eqref{empirical-threshold}, we know that $\hat{\alpha}_S \geq \alpha$. Hence, it suffices to establish that $\hat{\mathcal{H}}_{S}(\alpha) \cap \hat{\mathcal{H}}^*_{\alpha,T} \neq \emptyset$. Under the considered event, we have  
\[
\hat{\mathcal{H}}_{S}(\alpha) \supseteq \mathcal{H}_{S}(\alpha) \supseteq \mathcal{H}_{S}(\alpha_S),
\]  
so by \eqref{alpha_s_definition} there exists $h^*_{T,\alpha} \in T^*(\alpha)$ such that $h^*_{T,\alpha} \in \hat{\mathcal{H}}_{S}(\alpha)$. On the other hand, $h^*_{T,\alpha}\in \hat{\mathcal{H}}^*_{\alpha,T}$, since   \begin{align}\label{eq9}
\hat{R}_{\varphi,\mu_{0,T}}(h^*_{T,\alpha})\leq R_{\varphi,\mu_{0,T}}(h^*_{T,\alpha})+\epsilon_{0,T}\leq \alpha+\epsilon_{0,T} 
\end{align}
and 
\begin{align}\label{eq10}
\hat{R}_{\varphi,\mu_{1,T}}(h^*_{T,\alpha})\leq R_{\varphi,\mu_{1,T}}(h^*_{T,\alpha})+\epsilon_{1,T}&\leq R_{\varphi,\mu_{1,T}}(h^*_{T,\alpha-\epsilon_{0,T}})+\epsilon_{1,T}\nonumber\\
&\leq \hat{R}_{\varphi,\mu_{1,T}}(h^*_{T,\alpha-\epsilon_{0,T}})+2\epsilon_{1,T},
\end{align}
which implies that
\begin{align}\label{eq1}
    h^*_{T,\alpha} \in \hat{\mathcal{H}}_{S}(\alpha) \cap \hat{\mathcal{H}}^*_{\alpha,T} 
    \quad \text{and} \quad 
    \hat{\mathcal{H}}_{S}(\alpha) \cap \hat{\mathcal{H}}^*_{\alpha,T} \neq \emptyset .
\end{align}
Regarding the $\varphi$-Type-I error bound in Theorem \ref{theorem_bound}, we have $\hat{h}\in\hat{\mathcal{H}}'_{1,T}\subseteq \hat{\mathcal{H}}'=\hat{\mathcal{H}}_{S}(\alpha) \cap \hat{\mathcal{H}}^*_{\alpha,T}$, which implies that  
\[
R_{\varphi,\mu_{0,S}}(\hat{h}) \leq \alpha + 2\epsilon_{0,S} 
\quad \text{and} \quad 
R_{\varphi,\mu_{0,T}}(\hat{h}) \leq \alpha + 2\epsilon_{0,T}.
\]  
Together with Definition~\ref{def_Transfer_Modulus}, this yields the bound stated in Theorem~\ref{theorem_bound}.

Regarding the $\varphi$-Type-II excess error bound in Theorem \ref{theorem_bound}, first consider the case $h^*_{S,T,\alpha}\notin \hat{\mathcal{H}}'$. Since $h^*_{S,T,\alpha}\in \mathcal{H}_{S}(\alpha)\subseteq\hat{\mathcal{H}}_S(\alpha)$ and $h^*_{S,T,\alpha}\in \mathcal{H}_{T}(\alpha)\subseteq\hat{\mathcal{H}}_T(\alpha)$, it follows from \eqref{equation_H_Star} that
\begin{align}\label{eq15}
\hat{R}_{\varphi,\mu_{1,T}}(h^*_{S,T,\alpha})> 
\hat{R}_{\varphi,\mu_{1,T}}(\hat{h}^*_{T,\alpha-\epsilon_{0,T}})
+6\epsilon_{1,T}.
\end{align}
Thus,
\begin{align}\label{eq16}
\epsilon_{1,T}+R_{\varphi,\mu_{1,T}}(h^*_{S,T,\alpha})\geq \hat{R}_{\varphi,\mu_{1,T}}(h^*_{S,T,\alpha})& > \hat{R}_{\varphi,\mu_{1,T}}(\hat{h}^*_{T,\alpha-\epsilon_{0,T}})
+6\epsilon_{1,T}\nonumber\\
&\geq R_{\varphi,\mu_{1,T}}(\hat{h}^*_{T,\alpha-\epsilon_{0,T}})
+5\epsilon_{1,T}\nonumber\nonumber\\
&\geq R_{\varphi,\mu_{1,T}}(h^*_{T,\alpha})
+5\epsilon_{1,T}
\end{align}
where the last inequality uses the fact that \[R_{\varphi,\mu_{0,T}}(\hat{h}^*_{T,\alpha-\epsilon_{0,T}})\leq \hat{R}_{\varphi,\mu_{0,T}}(\hat{h}^*_{T,\alpha-\epsilon_{0,T}})+\epsilon_{0,T}\leq \alpha
\]
which implies $\hat{h}^*_{T,\alpha-\epsilon_{0,T}}\in \mathcal{H}_{T}(\alpha)$, together with the definition $h^*_{T,\alpha}=\argmin\limits_{h\in \mathcal{H}_{T}(\alpha)}R_{\varphi,\mu_{1,T}}(h)$. Therefore,
\begin{align}\label{eq2}
    R_{\varphi,\mu_{1,T}}(h^*_{S,T,\alpha})-R_{\varphi,\mu_{1,T}}(h^*_{T,\alpha})> 4\epsilon_{1,T}
\end{align}
On the other hand, for every $h\in\hat{\mathcal{H}}'_{1,T}$ we have
\begin{align*}
    R_{\varphi,\mu_{1,T}}(h)\leq \hat{R} _{\varphi,\mu_{1,T}}(h)+\epsilon_{1,T}\leq \hat{R}^*_{\varphi,\mu_{1,T}}(\hat{\mathcal{H}}')+3\epsilon_{1,T}&\stackrel{(1)}{\leq} \hat{R}_{\varphi,\mu_{1,T}}(h^*_{T,\alpha})+3\epsilon_{1,T}\\
    &\leq R_{\varphi,\mu_{1,T}}(h^*_{T,\alpha})+4\epsilon_{1,T},
\end{align*}
where inequality (1) holds because $h^*_{T,\alpha}\in \hat{\mathcal{H}}'$ by \eqref{eq1}, and $\hat{R}^*_{\varphi,\mu_{1,T}}(\hat{\mathcal{H}}')=\min\limits_{h\in\hat{\mathcal{H}}'} \hat{R}_{\varphi,\mu_{1,T}}(h).$ Hence,
\begin{align}\label{eq3}
\forall h \in \hat{\mathcal{H}}'_{1,T}, \quad \mathcal{E}_{1,T}(h) \leq 4\epsilon_{1,T}.
\end{align}
Combining this with \eqref{eq2}, we conclude the $\varphi$-Type-II excess error bound in Theorem \ref{theorem_bound} for the case where $h^*_{S,T,\alpha}\notin\hat{\mathcal{H}}'$.

Next, consider the case $h^*_{S,T,\alpha} \in \hat{\mathcal{H}}'$ . We first analyze the case where $\hat{\mathcal{H}}'_{1,S} \cap \hat{\mathcal{H}}'_{1,T} \neq \emptyset$,  
in which the learner selects $\hat{h} \in \hat{\mathcal{H}}'_{1,S} \cap \hat{\mathcal{H}}'_{1,T}$. By \eqref{eq3}, we know that $\mathcal{E}_{1,T}(\hat{h}) \leq 4\epsilon_{1,T}$. Moreover, since $\hat{h}\in \hat{\mathcal{H}}'_{1,S}$, we have
\begin{align}\label{eq30}
R_{\varphi,\mu_{1,S}}(\hat{h})\leq \hat{R}_{\varphi,\mu_{1,S}}(\hat{h})+\epsilon_{1,S}&\leq \hat{R}^*_{\varphi,\mu_{1,S}}(\hat{\mathcal{H}}')+3\epsilon_{1,S}\nonumber\\
    &\stackrel{(1)}{\leq} \hat{R}_{\varphi,\mu_{1,S}}(h^*_{S,T,\alpha})+3\epsilon_{1,S}\nonumber\\
    &\leq R_{\varphi,\mu_{1,S}}(h^*_{S,T,\alpha})+4\epsilon_{1,S}
\end{align}
where inequality (1) holds because $h^*_{S,T,\alpha}\in \hat{\mathcal{H}}'$, and $\hat{R}^*_{\varphi,\mu_{1,S}}(\hat{\mathcal{H}}')=\min\limits_{h\in\hat{\mathcal{H}}'} \hat{R}_{\varphi,\mu_{1,S}}(h).$ Hence,
\begin{align}\label{eq31}
\mathcal{E}_{1,S}(h|h^*_{S,T,\alpha})=R_{\varphi,\mu_{1,S}}(\hat{h})-R_{\varphi,\mu_{1,S}}(h^*_{S,T,\alpha})\leq 4\epsilon_{1,S},
\end{align}
which, together with Definition \ref{def_Transfer_Modulus}, implies  
\begin{align}\label{eq6}
    \mathcal{E}_{1,T}(\hat{h})\leq R_{\varphi,\mu_{1,T}}(h^*_{S,T,\alpha})-R_{\varphi,\mu_{1,T}}(h^*_{T,\alpha})+\phi_{1}^{S\to T}(4\epsilon_{1,S}).
\end{align}
This establishes the $\varphi$-Type-II excess error bound in Theorem~\ref{theorem_bound}. In the case $\hat{\mathcal{H}}'_{1,S} \cap \hat{\mathcal{H}}'_{1,T} = \emptyset$, the learner selects $\hat{h}=\argmin\limits_{h\in\hat{\mathcal{H}}'} \hat{R}_{\varphi,\mu_{1,T}}(h).$ Let $\hat{h}'_S$ denote a function belonging to the set $\hat{\mathcal{H}}'_{1,S}$. Since $\hat{\mathcal{H}}'_{1,S} \cap \hat{\mathcal{H}}'_{1,T} = \emptyset$, we obtain
\begin{align}\label{eq5}
\hat{R}_{\varphi,\mu_{1,T}}(\hat{h}'_S)>\hat{R}_{\varphi,\mu_{1,T}}(\hat{h})+2\epsilon_{1,T}.
\end{align}
We claim that $R_{\varphi,\mu_{1,T}}(\hat{h}'_S)>R_{\varphi,\mu_{1,T}}(\hat{h}).$ Otherwise, we would have
\[
\hat{R}_{\varphi,\mu_{1,T}}(\hat{h}'_S)\leq R_{\varphi,\mu_{1,T}}(\hat{h}'_S)+\epsilon_{1,S}\leq R_{\varphi,\mu_{1,T}}(\hat{h})+\epsilon_{1,S}\leq \hat{R}_{\varphi,\mu_{1,T}}(\hat{h})+2\epsilon_{1,S}
\]
which contradicts \eqref{eq5}. Therefore, 
\begin{align}\label{eq40}
    \mathcal{E}_{1,T}(\hat{h})&=R_{\varphi,\mu_{1,T}}(\hat{h})-R_{\varphi,\mu_{1,T}}(h^*_{T,\alpha})\nonumber\\
    &< R_{\varphi,\mu_{1,T}}(\hat{h}'_S)-R_{\varphi,\mu_{1,T}}(h^*_{T,\alpha})\nonumber\\
    &\leq R_{\varphi,\mu_{1,T}}(h^*_{S,T,\alpha})-R_{\varphi,\mu_{1,T}}(h^*_{T,\alpha})+\phi_{1}^{S\to T}(4\epsilon_{1,S}).
\end{align}
where the last inequality follows from \eqref{eq6}, which holds for every $h \in \hat{\mathcal{H}}'_{1,S}$. Hence, together with $\mathcal{E}_{1,T}(\hat{h}) \leq 4\epsilon_{1,T}$, we conclude the $\varphi$-Type-II excess error bound in Theorem~\ref{theorem_bound}.

\textbf{Case II:} $\alpha < \alpha_S.$ \\[6pt]
First, we show that $\hat{\alpha}_S \leq \alpha_S$. Since $\alpha_S > \alpha$, by \eqref{empirical-threshold} it suffices to show that $
\hat{\mathcal{H}}_{S}(\alpha_S) \cap \hat{\mathcal{H}}^*_{\alpha,T} \neq \emptyset.$ By \eqref{alpha_s_definition}, there exists $h^*_{T,\alpha} \in T^*(\alpha)$ such that 
$h^*_{T,\alpha} \in \mathcal{H}_S(\alpha_S) \subseteq \hat{\mathcal{H}}_S(\alpha_S)$. 
Moreover, using \eqref{eq9} and \eqref{eq10}, we have $
h^*_{T,\alpha} \in \hat{\mathcal{H}}^*_{\alpha,T}$, which completes the argument.

Regarding the $\varphi$-Type-I error bound in Theorem \ref{theorem_bound}, we have $\hat{h}\in\hat{\mathcal{H}}'_{1,T}\subseteq \hat{\mathcal{H}}'=\hat{\mathcal{H}}_{S}(\hat{\alpha}_S) \cap \hat{\mathcal{H}}^*_{\alpha,T}$, which implies that  
\[
R_{\varphi,\mu_{0,S}}(\hat{h}) \leq \hat{\alpha}_S + 2\epsilon_{0,S} 
\quad \text{and} \quad 
R_{\varphi,\mu_{0,T}}(\hat{h}) \leq \alpha + 2\epsilon_{0,T}.
\]  
Together with Definition~\ref{def_Transfer_Modulus}, this yields the bound stated in Theorem~\ref{theorem_bound}.

Regarding the $\varphi$-Type-II excess error bound in Theorem \ref{theorem_bound}, first consider the case $h^*_{S,T,\alpha}\notin \hat{\mathcal{H}}'$. Since $h^*_{S,T,\alpha}\in \mathcal{H}_{S}(\alpha)\subseteq \mathcal{H}_{S}(\hat{\alpha}_S)\subseteq\hat{\mathcal{H}}_S(\hat{\alpha}_S)$ and $h^*_{S,T,\alpha}\in \mathcal{H}_{T}(\alpha)\subseteq\hat{\mathcal{H}}_T(\alpha)$, it follows, by the same reasoning as in \eqref{eq15} and \eqref{eq16}, that \eqref{eq2} holds. On the other hand, for all $h \in \hat{\mathcal{H}}^*_{\alpha,T}$ we have
\begin{align*}
R_{\varphi,\mu_{1,T}}(h)\leq \hat{R}_{\varphi,\mu_{1,T}}(h)+\epsilon_{1,T}&\leq \hat{R}_{\varphi,\mu_{1,T}}(\hat{h}^*_{T,\alpha-\epsilon_{0,T}})+7\epsilon_{1,T}\\
&\leq \hat{R}_{\varphi,\mu_{1,T}}(h^*_{T,\alpha-2\epsilon_{0,T}})+7\epsilon_{1,T}\\
&\leq R_{\varphi,\mu_{1,T}}(h^*_{T,\alpha-2\epsilon_{0,T}})+8\epsilon_{1,T}\\
&\leq R_{\varphi,\mu_{1,T}}(h^*_{T,\alpha})+c\cdot\epsilon_{1,T}
\end{align*}
for some numerical constant $c$, where in the last inequality we used Lemma~\ref{Lemma_Type-II-error-control} together with the assumptions $\epsilon_{0,T} \leq \tfrac{7\alpha}{16}$ and $n_{0,T} \geq n_{1,T}$. Therefore, for all $h \in \hat{\mathcal{H}}'_{1,T} \subseteq \hat{\mathcal{H}}' \subseteq \hat{\mathcal{H}}^*_{\alpha,T}$ we obtain  \begin{align}\label{eq20}
\mathcal{E}_{1,T}(h)\leq c\cdot \epsilon_{1,T}
\end{align}
which, together with \eqref{eq2}, yields the bound stated in Theorem~\ref{theorem_bound} for the case where $h^*_{S,T,\alpha}\notin\hat{\mathcal{H}}'$.

Next, consider the case $h^*_{S,T,\alpha} \in \hat{\mathcal{H}}'$ . We first analyze the case where $\hat{\mathcal{H}}'_{1,S} \cap \hat{\mathcal{H}}'_{1,T} \neq \emptyset$,  
in which the learner selects $\hat{h} \in \hat{\mathcal{H}}'_{1,S} \cap \hat{\mathcal{H}}'_{1,T}$. By \eqref{eq20}, we have $\mathcal{E}_{1,T}(h)\leq c\cdot \epsilon_{1,T}$. Furthermore, since $\hat{h}\in\hat{\mathcal{H}}'_{1,S}$ and $h^*_{S,T,\alpha}\in\hat{\mathcal{H}}'$, we can apply the same reasoning as in \eqref{eq30} and \eqref{eq31} to obtain \eqref{eq6}, which in turn establishes the $\varphi$-Type-II excess error bound in Theorem~\ref{theorem_bound}.

In the case $\hat{\mathcal{H}}'_{1,S} \cap \hat{\mathcal{H}}'_{1,T} = \emptyset$, the learner selects $\hat{h}=\argmin\limits_{h\in\hat{\mathcal{H}}'} \hat{R}_{\varphi,\mu_{1,T}}(h).$ The same reasoning as in the case $\alpha \geq \alpha_S$ yields \eqref{eq40}, and together with \eqref{eq20}, this implies the $\varphi$-Type-II excess error bound in Theorem~\ref{theorem_bound}.





\section{Proof of Theorem~\ref{thm:opt rate}}

In this section we will present the proof of Theorem~\ref{thm:opt rate}, the convergence rate of our optimization procedure. We first introduce the following Theorem that establishes the convergence of our convex program solver (Algorithm~\ref{algorithm:CP}), which is an important component of our main algorithm.

\begin{theorem}\label{thm:cpsolver rate}
Let $\theta^* = \arg\min_{\theta} f(\theta) \ \text{s.t.} \ \tilde g(\theta) \leq 0$ and $\rho^2 \doteq B^2 + \frac{G^2}{r^2} $. Assume $f$ and $g$ are convex, $G$ Lipschitz, $H$ smooth and $\sup_{\theta,\zeta}|g(\theta;\zeta)|\leq C_g$. Suppose $\epsilon_0 \leq \frac{C_g \sqrt{3\log\frac{2}{\delta}}}{\pare{\lambda^* + \sqrt{2}\rho  }\sqrt{N}}$.  For Algorithm~\ref{algorithm:CP}, assuming $g(\theta) = \tilde g(\theta)- c$ for some $c\in [0,\epsilon_0]$,  if we choose $\eta  = \frac{c_\eta}{\sqrt{N}}$, where
 {\small\begin{align*}
     c_{\eta} \leq \min\cbr{ \frac{\rho}{2\sqrt{3} C_{g} }, \frac{ \rho}{ 12\sqrt{6} (1+\sqrt{2}\rho + \lambda^*)  G \sqrt{3\log\frac{2}{\delta}} }, \frac{\rho}{ 12\sqrt{6} C_g \sqrt{ \log\frac{2}{\delta}} } , \frac{\rho}{ 32\sqrt{N}\epsilon_0},\frac{\sqrt{G} }{ r-2H\epsilon_0 }  }, 
 \end{align*}} then with probability at least $1-\delta$ we have
    \begin{align*}
        &f(\hat{\theta} )- f(\theta^* )  \lesssim \pare{ G   +    C_g \sqrt{ \log\frac{1}{\delta}}  }\pare{  \frac{   G  C_g   \sqrt{ \log\frac{1}{\delta}} +     \frac{G }{c_\eta}  }{\sqrt{N}(r-2H\epsilon_0)^2 }      +  \frac{\xi + \epsilon_0}{  (r-2H\epsilon_0) } + \frac{\frac{\rho^2}{c_{\eta}}  +\frac{G}{r} C_g  \sqrt{ \log\frac{1}{\delta}}}{\sqrt{N}} } .
         \end{align*}
That is, if we choose $\xi = \epsilon_0$ and $\epsilon_0 =\min\cbr{ \frac{r\epsilon}{1+H\epsilon} \pare{G   +   C_g  \sqrt{ \log\frac{1}{\delta}}  }^{-1}, \frac{r}{4H} }$, we know $f(\hat{\theta} )- f(\theta^* ) \leq \epsilon$ and the number of total gradient evaluations is bounded by
\begin{align*}
 N \gtrsim   \pare{ G  +  C_g  \sqrt{ \log\frac{1}{\delta}}  }\pare{  \frac{ G  C_g   \sqrt{ \log\frac{1}{\delta}} +     \frac{ G }{c_\eta}  }{\epsilon^2 r^2 }     + \frac{\frac{\rho^2}{c_{\eta}}  +\frac{G}{r}  C_g   \sqrt{ \log\frac{1}{\delta}}}{\epsilon^2 } } .
\end{align*}
\end{theorem}



To prove Theorem~\ref{thm:cpsolver rate}, we need the following lemma.
 
\begin{lemma}\label{lemma:convergence of lag}
    For Algorithm~\ref{algorithm:CP},  the following statement holds true for any $\lambda \geq 0$ with probability at least $1-\delta$:
\begin{align*}
        &\left(f(\bar\theta_N)- f(\theta^* )\right) +    \lambda \tilde g(\bar\theta_T)   -\pare{\frac{\gamma}{2} + \frac{1}{2\eta N}}\lambda^2        \\
            &  \leq   \frac{\rho^2}{\eta N}      +   \eta C_g^2   +   \eta G^2  +   \frac{C_g  \sqrt{3\log\frac{2}{\delta}} }{\sqrt{N}}  (\lambda^* + \sqrt{2}\rho  )+\frac{\sqrt{2}\rho  (1+\sqrt{2}\rho + \lambda^*) G\sqrt{3\log\frac{2}{\delta}}}{\sqrt{N}}  \\
          & +  \epsilon_0 (\lambda^* + \sqrt{2}\rho  )+  \pare{\frac{C_g   \sqrt{3\log\frac{2}{\delta}}}{\sqrt{N}}    + \epsilon_0} \lambda .
        \end{align*}
    \end{lemma}
    \begin{proof}
Define $L(\theta,\lambda) \doteq f(\theta) + \lambda \tilde g(\theta)- \frac{\gamma\lambda^2}{2}$, and let $\lambda^* = \arg\max_{\lambda\geq 0} \min f(\theta) + \lambda \tilde g(\theta)$.
 We first show that $\norm{\theta_t - \theta^*}^2 + \norm{\lambda_t - \lambda^*}^2 \leq 2\rho^2$, for any $t \in [T]$. We prove this by induction. Assume this holding for $t$, and for $t+1$ we have
        \begin{align*}
            \norm{\theta_{t+1} - \theta^* }^2 &= \norm{ \theta_t - \eta g_t - \theta^* }^2\\ 
            &= \norm{  \theta_t - \theta^*}^2 - 2\inprod{\eta g_t}{ \theta_t - \theta^*} + \eta^2 \norm{ g_t }^2,
        \end{align*}
       where $g_t =  \nabla f(\theta_t;\zeta_{f,t}) + \lambda_t \nabla g(\theta_t;\zeta_{g,t})$.

    Then for $t+1$,  we have:
       \begin{align*}
            \norm{\theta_{t+1} -  \theta^*   }^2   
            & \leq  \norm{  \theta_t - \theta^*  }^2 - 2\eta_t \inprod{ \nabla f(\theta_t ) +\lambda_t \nabla g(\theta_t )}{ \theta_t -  \theta^*  }  + 2\sqrt{2}\eta  p_t \rho   + \eta_t^2 (1+\lambda_t)^2 G^2 \\ 
            &\leq   \norm{  \theta_t -  \theta^*  }^2 - 2\eta  \pare{ L(\theta_t,\lambda_t) - L(\theta^* ,\lambda_t) }   + 2\sqrt{2}\eta p_t \rho + \eta^2 (1+\lambda_t)^2 G^2 
        \end{align*} 
            where $p_t = \norm{\nabla f(\theta_t  ) +\lambda_t \nabla g(\theta_t  )- g_t} $, and at last step we use the convexity of $L(\cdot,\lambda_t)$.  We know that $$\E[g_t] = \nabla f(\theta_t ) +\lambda_t \nabla g(\theta_t ) , \E[\exp\pare{p_t^2/ (G^2+\lambda_t G^2 )}] \leq \exp(1)$$. 
     Similarly,  we have:
       \begin{align*}
            |\lambda_{t+1} - \lambda^*   |^2 & =  |\lambda_{t} - \lambda^* |^2 + 2 \eta \inprod{ g(\theta_t;\zeta_{g,t}) -\gamma \lambda_t }{\lambda_{t} - \lambda^* } + \eta^2  | g(\theta_t;\zeta_{g,t})- \gamma \lambda_t  |^2\\
             & \leq (1-\gamma\eta) |\lambda_{t} - \lambda |^2 - 2 \eta \pare{L(\theta_t, \lambda^* ) - L(\theta_t, \lambda_t)   }  \\
             &\quad + 2\sqrt{2}\eta q_t \rho + 2\sqrt{2}\eta h_t \rho +  2\eta^2  C_g^2,
        \end{align*}
        where $$q_t = | g(\theta_t;\zeta_{g,t})-  g( \theta_{ t}   ) |,h_t = | g( \theta_{t} ) - \tilde g( \theta_{t}) | = \epsilon_0 $$ and at last step we use the $\gamma$-strong-concavity of $L(\theta_t,\cdot)$. It is easy to see that $$\E[g(\theta_t;\zeta_{g,t})] = g(\theta_{t}  ) , \E[\exp\pare{q_t^2/C_g^2}] \leq \exp(1).$$

    Putting pieces together we have
      \begin{align*}
          \norm{\theta_{t+1} - \theta^*  }^2+   |\lambda_{t+1} - \lambda^*  |^2& \leq     \pare{|\lambda_{t} - \lambda^* |^2 + \norm{  \theta_t - \theta^* }^2 }   - 2 \eta  \pare{L(\theta_t, \lambda^*) - L(\theta^*  , \lambda_t)   } \\
             & + 2\sqrt{2}\eta p_t \rho  + 2\sqrt{2}\eta q_t \rho + 2 \sqrt{2}\eta \epsilon  \rho +  2\eta^2 C_g^2  + \eta^2 (1+\lambda_t)^2 G^2 \\
              &\leq  \pare{|\lambda_{t} - \lambda^* |^2 + \norm{  \theta_t - \theta^* }^2 }  - 2 \eta  \underbrace{\pare{L(\theta_t, \lambda^*) - L(\theta^*  , \lambda_t)   }  }_{\geq - \frac{\gamma (\lambda^*)^2}{2}}  \\
              &+  2\eta^2 C_g^2 + \eta^2 (1+\sqrt{2}\rho + \lambda^*)^2 G^2 + 2\sqrt{2}\eta \rho ( p_t  +  q_t   +  \epsilon_0 )
        \end{align*}
        where the last step is due to
    \begin{align*}
        & L(\theta_t, \lambda^*) - L(\theta^*  , \lambda_t)  = \underbrace{f(\theta_t) + \lambda^* \tilde g(\theta_t) -   \pare{  f(\theta^*) + \lambda_t \tilde g(\theta^*)}}_{\geq 0} - \frac{\gamma (\lambda^*)^2}{2} + \frac{\gamma \lambda_t^2}{2}.
    \end{align*}
        
    Performing telescoping sum yields:
    \begin{align*}
          \norm{\theta_{t+1} - \theta^*  }^2+   |\lambda_{t+1} - \lambda^*  |^2  
              &\leq  \pare{ \norm{  \theta_0 - \theta^* }^2+|\lambda_{0} - \lambda^* |^2  }     +  2\eta^2 C_{g}^2  + \eta^2 (1+\sqrt{2}\rho + \lambda^*)^2  G^2 \\
             &\quad + 2\sqrt{2}\eta \rho \sum_{s=0}^t p_s    + 2\sqrt{2}\eta \rho \sum_{s=0}^t q_s + 2\sqrt{2} \eta \rho t\epsilon_0 + N \gamma \eta (\lambda^*)^2  .
        \end{align*} 
        Due to Lemma~4 of~\citep{lan2012optimal}, we know with probability $1-\delta/2$,
        \begin{align}
               \sum_{t=0}^{N-1} p_t  \leq (1 +\lambda_t  )G \sqrt{N  } \sqrt{3\log\frac{2}{\delta}} , \sum_{t=0}^{N-1} q_t  \leq  \sqrt{N  C_{g}^2} \sqrt{3\log\frac{2}{\delta}} , \label{eq:martingale bound}
        \end{align}
         
  Putting pieces together yields:
        \begin{align}
          \norm{\theta_{t+1} - \theta^*  }^2 & + |\lambda_{t+1} - \lambda^*  |^2  \leq   |\lambda_{0} - \lambda^* |^2 + \norm{  \theta_0 - \theta^* }^2    +  2\eta^2  C_{g}^2  + \eta^2 (1+\sqrt{2}\rho + \lambda^*)^2  G^2 \nonumber \\
          &  + 2\sqrt{2}\eta \rho \sqrt{N}(1+\sqrt{2}\rho + \lambda^*)G \sqrt{3\log\frac{2}{\delta}} \nonumber \\
          & + 2\sqrt{2}\eta \rho \sqrt{N}  C_{g}    \sqrt{3 \log\frac{2}{\delta}}    +  2\sqrt{2}\eta \rho  N\epsilon_0 + T\gamma \eta (\lambda^*)^2.  \label{eq:convergence iteration t}
        \end{align}
 Since we choose $\eta = \frac{c_{\eta}}{\sqrt{N}}$ and $\gamma =  G^2 \eta$, where
 \begin{align*}
     c_{\eta} \leq \min\cbr{ \frac{\rho}{2\sqrt{3} C_{g} }, \frac{ \rho}{ 12\sqrt{6} (1+\sqrt{2}\rho + \lambda^*)  G \sqrt{3\log\frac{2}{\delta}} }, \frac{\rho}{ 12\sqrt{6} C_g \sqrt{ \log\frac{2}{\delta}} } , \frac{\rho}{ 32\sqrt{N}\epsilon}  },
 \end{align*}
 we conclude that $\norm{\theta_{t+1} - \theta^*  }^2+   |\lambda_{t+1} - \lambda^*  |^2 
 \leq 2\rho^2$.  
 
 Now  by similar analysis we have that for any $\lambda \geq 0$  
        \begin{align*}
     \norm{\theta_{t+1} - \theta^*  }^2+   |\lambda_{t+1} - \lambda  |^2        & \leq  \norm{  \theta_t - \theta  }^2  + |\lambda_{t} - \lambda  |^2    -2\eta(  L(\theta_t, \lambda ) - L(\theta   , \lambda_t)  )   \\
     &  +  2\eta^2   C_g^2   + \eta^2 (1+  \lambda_t)^2   G^2  + 2\eta  \inprod{ g(\theta_t;\zeta_{g,t})  -  g(\theta_{t}) }{\lambda_t - \lambda}\\
          &  + 2\eta  \inprod{   g(\theta_{t})  - \tilde g(\theta_{t})   }{\lambda_t - \lambda} + 2\eta p_t \norm{\theta_t - \theta^*}  \\
          & \leq  \norm{\theta_t - \theta^* }^2  + |\lambda_{t} - \lambda  |^2   -2\eta(  L(\theta_t, \lambda ) - L(\theta^*  , \lambda_t)  )\\
          &+  2\eta^2 C_g^2  + \eta^2 (1+  \lambda_t)^2   G^2     + 2\eta  q_t \pare{\lambda_t + \lambda} \\
          &+ 2\eta \epsilon_0 \pare{\lambda_t + \lambda} + 2\sqrt{2}\eta p_t \rho  .
        \end{align*}
    Since $|\lambda_t - \lambda^*| \leq \sqrt{2}\rho$ we know $\lambda_t \leq \lambda^* + \sqrt{2}\rho$. Hence we have
         \begin{align*}
     \norm{\theta_{t+1} - \theta^*  }^2+   |\lambda_{t+1} - \lambda  |^2   & \leq   |\lambda_{t} - \lambda  |^2 + \norm{  \theta_t - \theta^* }^2   -2\eta(  L(\theta_t, \lambda ) - L(\theta^*  , \lambda_t)  )\\
     &+  2\eta^2 C^2_{g}  + \eta^2 (1+  \lambda_t)^2  G^2+ 2\eta  q_t \pare{\lambda^* + \sqrt{2}\rho  }   \\
          &+ 2\eta \epsilon_0 \pare{\lambda^* + \sqrt{2}\rho  }  + 2\eta  q_t \lambda  + 2\eta  \epsilon_0 \lambda  + 2\sqrt{2}\eta p_t \rho .
        \end{align*}
      Performing telescoping sum yields:
      \begin{align*}
        \frac{1}{N} \sum_{t=0}^{N-1}   L(\theta_t, \lambda ) - L(\theta^*  , \lambda_t)        & \leq  \frac{1}{2\eta N} (|\lambda_{0} - \lambda  |^2 + \norm{  \theta_0 - \theta  }^2    )  +   \frac{1}{N}\eta C_g^2  + \frac{1}{2N}\eta  \sum_{t=0}^{N-1} (1+\lambda_t)^2  G^2  \\
          &  +   \frac{1}{N} \sum_{t=0}^{N-1} q_t \pare{\lambda^* + \sqrt{2}\rho  } +  \frac{1}{N} \sum_{t=0}^{N-1}  \epsilon_0 \pare{\lambda^* + \sqrt{2}\rho  } + \frac{1}{N} \sum_{t=0}^{N-1}  q_t \lambda \\
          &+ \epsilon_0  \lambda  +  \sqrt{2}\rho \frac{1}{N} \sum_{t=0}^{N-1} p_t.    
        \end{align*}

        By the definition of Lagrangian, we have 
        \begin{align*}
      &  \frac{1}{N}  \sum_{t=0}^{N-1}  (f(\theta_t)+ \lambda \tilde g(\theta_t)-\frac{\gamma}{2}\lambda^2  - f(\theta^* )  - \lambda_t \underbrace{ \tilde g( \theta^* )  }_{\leq 0} + \frac{\gamma}{2}\lambda^2_t  )  \\
          &  \leq    \frac{1}{2\eta N} (|\lambda_{0} - \lambda  |^2 + \norm{  \theta_0 - \theta^* }^2    )  +   \frac{1}{N}\eta C_g^2  + \frac{1}{2N}\eta  \sum_{t=0}^{N-1} (1+\lambda_t)^2 G^2  \\
          &+   \frac{1}{N} \sum_{t=0}^{N-1} q_t \pare{\lambda^* + \sqrt{2}\rho  } +   \epsilon_0 \pare{\lambda^* + \sqrt{2}\rho  } + \frac{1}{N} \sum_{t=0}^{N-1}  q_t \lambda  +  \epsilon_0  \lambda  +  \sqrt{2}\rho \frac{1}{N} \sum_{t=0}^{N-1} p_t  .
        \end{align*}
        Evoking the bound from (\ref{eq:martingale bound})   yields:
         \begin{align*}
       \frac{1}{N}  \sum_{t=0}^{N-1} & (f(\theta_t) + \lambda  \tilde g(\theta_t)-\frac{\gamma}{2}\lambda^2  - f(\theta^* )    + \frac{\gamma}{2}\lambda^2_t   )  \\
          &  \leq    \frac{1}{2\eta N} (|\lambda_{0} - \lambda  |^2 + \norm{  \theta_0 - \theta^* }^2    )  +   \frac{1}{N}\eta C_g^2  + \frac{1}{2N}\eta \sum_{t=0}^{N-1}  (1+\lambda_t)^2 G^2    \\
          &  +   \frac{1}{N} \sqrt{TC_g^2} \sqrt{3\log\frac{2}{\delta}} \pare{\lambda^* + 2\sqrt{2}\rho  } +  \epsilon_0 \pare{\lambda^* + \sqrt{2}\rho  }  \\
          &+ \frac{1}{N} \sqrt{TC_g^2} \sqrt{3\log\frac{2}{\delta}} \lambda  + \epsilon_0 \lambda+ \frac{\sqrt{2}\rho (1+\sqrt{2}\rho + \lambda^*)G  \sqrt{3\log\frac{2}{\delta}}}{\sqrt{N}} .
        \end{align*}
        
      Plugging in $\lambda_0 = 0$, $\theta_0 = \mathbf{0}$ and re-arranging the terms yields:
          \begin{align*}
        \frac{1}{N}  \sum_{t=0}^{N-1}  &(f(\theta_t)- f(\theta^* ) ) +   \frac{1}{N} \sum_{t=0}^{N-1} \lambda  \tilde g(\theta_t)-\pare{\frac{\gamma}{2} + \frac{1}{2\eta N}}\lambda^2        \\
          &  \leq   \frac{\rho^2}{\eta T}      +  \eta C_g^2  + \frac{1}{2N}\sum_{t=0}^{N-1} (\eta  (1+\lambda_t)^2 G^2 -  \gamma \lambda^2_t) +\frac{\sqrt{2}\rho  (1+\sqrt{2}\rho + \lambda^*) G \sqrt{3\log\frac{2}{\delta}}}{\sqrt{N}} \\
          &+   \frac{C_g  \sqrt{3\log\frac{2}{\delta}} }{\sqrt{N}} (\lambda^* + 2\sqrt{2}\rho )+ \epsilon_0   (\lambda^* + 2\sqrt{2}\rho ) +  \pare{\frac{C_g   \sqrt{3\log\frac{2}{\delta}}}{\sqrt{N}}    + \epsilon_0 } \lambda .
        \end{align*}
        
    By our choice, $\gamma = G^2 \eta$, so we have
         \begin{align*}
      &  \frac{1}{N}  \sum_{t=0}^{N-1}  \left(f(\theta_t)- f(\theta^* )\right) +   \frac{1}{N}  \sum_{t=0}^{N-1} \lambda  \tilde g(\theta_{t} )  -\pare{\frac{\gamma}{2} + \frac{1}{2\eta N}}\lambda^2        \\
          &  \leq   \frac{\rho^2}{\eta T}      +   \eta C_g^2   +   \eta G^2  +   \frac{C_g  \sqrt{3\log\frac{2}{\delta}} }{\sqrt{N}} (\lambda^* + \sqrt{2}\rho  )+\frac{\sqrt{2}\rho  (1+\sqrt{2}\rho + \lambda^*) G \sqrt{3\log\frac{2}{\delta}}}{\sqrt{N}}  \\
          &  +  \epsilon_0 (\lambda^* + \sqrt{2}\rho  )+  \pare{\frac{C_g   \sqrt{3\log\frac{2}{\delta}}}{\sqrt{N}}    + \epsilon_0 } \lambda .
        \end{align*}
        Define $\hat \theta_T = \frac{1}{N}\sum_{t=0}^{N-1}\theta_T$, and then by Jensen's inequality we have
\begin{align*}
        &\left(f(\bar\theta_T)- f(\theta^* )\right) +    \lambda \tilde g(\bar\theta_T)   -\pare{\frac{\gamma}{2} + \frac{1}{2\eta N}}\lambda^2        \\
            &  \leq   \frac{\rho^2}{\eta T}      +   \eta C_g^2   +   \eta G^2  +   \frac{C_g  \sqrt{3\log\frac{2}{\delta}} }{\sqrt{N}}  (\lambda^* + \sqrt{2}\rho  )+\frac{\sqrt{2}\rho  (1+\sqrt{2}\rho + \lambda^*) G\sqrt{3\log\frac{2}{\delta}}}{\sqrt{N}}  \\
          & +  \epsilon_0 (\lambda^* + \sqrt{2}\rho  )+  \pare{\frac{C_g   \sqrt{3\log\frac{2}{\delta}}}{\sqrt{N}}    + \epsilon_0} \lambda .
        \end{align*}
        \end{proof}

\subsection{Proof of Theorem~\ref{thm:cpsolver rate}} \label{app:proof of cpsolver}
        
   \begin{proof}
            Note that Lemma~\ref{lemma:convergence of lag} holds for any $\lambda \geq 0$.
        Now let's discuss by cases. If $\bar\theta_T$ is in the constraint set, then $\hat{\theta}  = \bar\theta_T$ and we simply set $\lambda = 0$ and get the convergence:
       \begin{align*}
        &\left(f(\bar\theta_T)- f(\theta^* )\right) +    \lambda \tilde g(\bar\theta_T)   -\pare{\frac{\gamma}{2} + \frac{1}{2\eta N}}\lambda^2        \\
            &  \leq   \frac{\rho^2}{\eta T}      +   \eta C_g^2   +   \eta G^2  +   \frac{C_g  \sqrt{3\log\frac{2}{\delta}} }{\sqrt{N}}  (\lambda^* + \sqrt{2}\rho  )+\frac{\sqrt{2}\rho  (1+\sqrt{2}\rho + \lambda^*) G\sqrt{3\log\frac{2}{\delta}}}{\sqrt{N}}  \\
          & +  \epsilon(\lambda^* + \sqrt{2}\rho  )+  \pare{\frac{C_g   \sqrt{3\log\frac{2}{\delta}}}{\sqrt{N}}    + \epsilon} \lambda .
        \end{align*}
        
        If $\bar\theta_T$ is not in the constraint set, we set  $\lambda = \frac{ \tilde g(\bar\theta_T)  }{\gamma + \frac{1}{\eta T}}$,  which yields:
        \begin{align}
         (f(\bar\theta_T)& - f(\theta^* ) ) +     \frac{(\tilde g(\bar\theta_T))^2}{2(\gamma + \frac{1}{\eta T})}     \nonumber    \\
          &  \leq   \frac{\rho^2}{\eta T}      +   \eta C_g^2   +   \eta G^2  +   \frac{C_g  \sqrt{3\log\frac{2}{\delta}} }{\sqrt{N}}  \pare{\lambda^* + 2\sqrt{2}\rho  } +\frac{\sqrt{2}\rho  (1+\sqrt{2}\rho + \lambda^*) G\sqrt{3\log\frac{2}{\delta}}}{\sqrt{N}}  \nonumber\\
          &  +  \epsilon_0 \pare{\lambda^* + \sqrt{2}\rho  } +  \pare{\frac{C_g   \sqrt{3\log\frac{2}{\delta}}}{\sqrt{N}}    + \epsilon }  \left| \frac{\tilde g(\bar\theta_T)  }{\gamma + \frac{1}{\eta T}}\right|. \label{eq:bound with g}
        \end{align}
         Since $\bar\theta_T$ is not in the constraint set and $\hat{\theta} $ is the projection of it onto inexact constraint set $\hat g(\theta)    \leq -\delta $, by KKT condition we know $ g(\hat{\theta} ) = -\delta$ and $\bar\theta_T  -  \hat{\theta} = s\cdot \nabla   g(\hat{\theta})$ for some $s > 0$ where $  g'(\theta)\doteq  g(\theta) + \delta$. Defining $\Delta := \delta - c$, and due to our choice of $T$ we know $\Delta \geq 0$. Then we have
         \begin{align*}
            \tilde  g(\bar\theta_T) &= \tilde  g(\bar\theta_T)  -   g'(\hat\theta)  \\
              &= \tilde g(\bar\theta_T)  -  \tilde g (\hat\theta) - (  g' (\bar\theta_T) - \tilde g(\bar\theta_T)) \\
               &= \tilde g(\bar\theta_T)  -  \tilde g (\hat\theta) - (    \delta - c) \\
              &\geq \inprod{\nabla \tilde  g(\hat\theta)}{\bar\theta_T  -  \hat\theta} -\Delta = \norm{\nabla   \tilde  g(\hat\theta)}\norm{\bar\theta_T  -  \hat{\theta} }- \Delta
         \end{align*}
         where the inequality is due to convexity of $\tilde g(\cdot)$. Let $\theta_0$ be such that $\tilde g(\theta_0) = 0$, and then we have
            \begin{align*}
                \min_{ g'(\theta) = 0}\norm{\nabla  \tilde g( {\theta} )} &\geq \min_{ g'(\theta) = 0} \norm{ \nabla \tilde  g(  \theta_0  ) } - \norm{\nabla  \tilde g( {\theta}_0 ) - \nabla  \tilde g( {\theta} )} \\
                &\geq r - 2L \pare{\tilde g( {\theta}_0 ) - \tilde g( {\theta} ) } = r - 2L    \tilde g( {\theta} )     \\
                &= r - 2L    (c-\delta )     \\
                &   \geq r - 2L  \epsilon_0  
            \end{align*}
            , so $g(\bar\theta_T) \geq ( r - 2L  \epsilon_0 )\norm{\bar\theta_T  -  \hat{\theta} } - \Delta$.
         On the other hand, 
         since $\hat{\theta} $ is the projection of $\bar\theta_T$ onto constraint set, and $\theta^*$ is in the constraint set, we know
          
        \begin{align*}
            \norm{\hat{\theta}  - \theta^*}^2 \leq \norm{\bar{\theta}_{T} - \theta^*}^2 \leq 2\rho^2.
        \end{align*}
         Hence
         we also know
         \begin{align*}
            \tilde g(\bar\theta_T) &= \tilde g(\bar\theta_T) -   g'(\hat{\theta} ) \\
             &= \tilde g(\bar\theta_T) -    g(\hat{\theta} ) - \delta\\
             & = \tilde g(\bar\theta_T) - \tilde g  (\hat{\theta} ) + \tilde g(\hat{\theta} ) -    g(\hat{\theta} ) - \delta\\
             &\leq   G\norm{\bar\theta_T  -  \hat{\theta} }  + \underbrace{c - \delta}_{\leq 0}.
         \end{align*}
         Plugging the upper and lower bound of $g(\bar\theta_T)$ into (\ref{eq:bound with g}) yields:
  \begin{align}
         (f(\bar\theta_T)& - f(\theta^* ) ) +     \sqrt{N} \frac{((r-2H\epsilon_0) \|\bar\theta_T  -  \hat{\theta} \|- \Delta)^2}{2(c_\eta G^2 + \frac{1}{c_\eta})}     \nonumber    \\
          &  \leq   \frac{\rho^2}{\eta T}      +   \eta C_g^2   +   \eta G^2  +   \frac{C_g  \sqrt{3\log\frac{2}{\delta}} }{\sqrt{N}}  \pare{\lambda^* + 2\sqrt{2}\rho  } +\frac{\sqrt{2}\rho  (1+\sqrt{2}\rho + \lambda^*) G\sqrt{3\log\frac{2}{\delta}}}{\sqrt{N}}  \nonumber\\
          &  +  \epsilon_0 \pare{\lambda^* + \sqrt{2}\rho  } +  \pare{\frac{C_g   \sqrt{3\log\frac{2}{\delta}}}{\sqrt{N}}    + \epsilon_0 }    \frac{ \sqrt{N}  }{2(c_\eta G^2 + \frac{1}{c_\eta})} \hat G\norm{\bar\theta_T  -  \hat{\theta} }. 
        \end{align}

        Notice the following decomposition: $
            f(\bar\theta_T)- f(\theta^* ) \geq f(\bar\theta_T)- f(\hat{\theta}  ) \geq -\hat G \norm{\bar\theta_T - \hat{\theta}  }.$
    Also notice the fact $(a-b)^2 \geq \frac{1}{2}a^2 - b^2$ holding for any $a>0, b>0$, we know 
    \begin{align*}
        \pare{(r-2H\epsilon_0)\norm{\bar\theta_T  -  \hat{\theta} }-\Delta}^2  \geq  \frac{1}{2}(r-2H\epsilon_0)^2\norm{\bar\theta_T  -  \hat{\theta} }^2 -\Delta^2.
    \end{align*}
    Putting pieces together yields the following inequality:
    \begin{align*}
       &a \norm{\bar\theta_T - \hat{\theta} }^2 -b \norm{\bar\theta_T - \hat{\theta} } - c \leq 0,\\
       \text{where:}&\\
       & a =\frac{ \sqrt{N}(r-2H\epsilon_0)^2}{4(c_\eta  G^2 + \frac{1}{c_\eta})}   \\
       & b = \frac{ G}{(c_\eta G^2 + \frac{1}{c_\eta})}   \pare{ C_g   \sqrt{3\log\frac{2}{\delta}}     + \sqrt{N}\epsilon_0 }    +   G,\\
       &c = \frac{\frac{\rho^2}{c_{\eta}} + c_{\eta}(  C_g^2 + G^2) + C_g \sqrt{3\log\frac{2}{\delta}} \pare{\lambda^* + 2\sqrt{2}\rho  }+\sqrt{2}\rho  (1+\sqrt{2}\rho + \lambda^*) G\sqrt{3\log\frac{2}{\delta}}}{\sqrt{N}}   \\
       &\quad +  \epsilon_0 \pare{\lambda^* + \sqrt{2}\rho  } +       \sqrt{N}\frac{    \Delta^2}{2(c_\eta G^2 + \frac{1}{c_\eta})}  .
    \end{align*}
     
    Assume $ \epsilon_0 \leq  \frac{C_g^2  {3\log\frac{2}{\delta}} }{\sqrt{N}} $,
    so $ b \leq   \frac{2  G C_g   \sqrt{3\log\frac{2}{\delta}} }{(c_\eta  G^2 + \frac{1}{c_\eta})}   +  G. $
    Solving the above quadratic inequality yields:
   { \small \begin{align*}
       & \norm{\bar\theta_T - \hat{\theta} } \leq \frac{b+\sqrt{b^2+4ac}}{2a} \leq  \frac{b}{a} + \sqrt{\frac{c}{a}}\\
        &\leq   \frac{2 G C_g   \sqrt{3\log\frac{2}{\delta}} }{\sqrt{N}(r-2H\epsilon_0)^2 }    + \frac{4 G (c_\eta G^2 + \frac{1}{c_\eta})}{ \sqrt{N}(r-2H\epsilon_0)^2} +  \frac{\Delta}{2 (r-2H\epsilon_0) }  \\
        &+\sqrt{\frac{4(c_\eta G^2 + \frac{1}{c_\eta})}{ \sqrt{N}(r-2H\epsilon_0)^2 } } \sqrt{\frac{\frac{\rho^2}{c_{\eta}} + c_{\eta}( C_g^2 + G^2) +   (C_g\pare{\lambda^* + 2\sqrt{2}\rho  }+\sqrt{2}\rho  (1+\sqrt{2}\rho + \lambda^*) G)\sqrt{3\log\frac{2}{\delta}}}{\sqrt{N}}    }\\   
        &+   \sqrt{\frac{4(c_\eta G^2 + \frac{1}{c_\eta})}{ \sqrt{N}(r-2H\epsilon_0)^2 } }  \epsilon_0 \pare{\lambda^* + \sqrt{2}\rho  }   \\
          &=  \frac{2 G C_g   \sqrt{3\log\frac{2}{\delta}} + 6 G (c_\eta  G^2 + \frac{1}{c_\eta}) }{\sqrt{N}(r-2H\epsilon_0)^2 }      +  \frac{\Delta}{2 (r-2H\epsilon_0) }   \\
        &   +  \frac{\frac{\rho^2}{c_{\eta}} + c_{\eta}(C_g^2 + G^2) + C_g \sqrt{3\log\frac{2}{\delta}} \pare{\lambda^* + 2\sqrt{2}\rho  }+\sqrt{2}\rho  (1+\sqrt{2}\rho + \lambda^*) G\sqrt{3\log\frac{2}{\delta}}}{2\sqrt{N}} +  \epsilon_0 \pare{\lambda^* + \sqrt{2}\rho  }  \\ 
    \end{align*}}
    where at the last step we use the fact $\sqrt{ab} \leq \frac{a^2+b^2}{2}$.
    Finally, note the following decomposition:
    \begin{align*}
      f(\hat{\theta} )-  f(\theta^* )  &=  f(\hat{\theta} )-f(\bar\theta_T) +f(\bar\theta_T) - f(\theta^* ) \\
        &\leq  G\norm{ \bar\theta_T - \hat{\theta}  } +f(\bar\theta_T) - f(\theta^* ) \\
         &\leq \pare{ G  +  \frac{ 2C_g   \sqrt{3\log\frac{2}{\delta}}}{c_\eta G^2 + \frac{1}{c_\eta}} G}\norm{ \bar\theta_T - \hat{\theta}  } +  \epsilon_0\pare{\lambda^* + \sqrt{2}\rho  }\\
         & + \frac{\frac{\rho^2}{c_{\eta}} + c_{\eta}(  C_g^2 + G^2) + (C_g  \pare{\lambda^* + 2\sqrt{2}\rho  }+\sqrt{2}\rho  (1+\sqrt{2}\rho + \lambda^*) G)\sqrt{3\log\frac{2}{\delta}}}{\sqrt{N}}   .
    \end{align*}
    
 Recalling $\epsilon_0 \leq \frac{C_g \sqrt{3\log\frac{2}{\delta}}}{\pare{\lambda^* + \sqrt{2}\rho  }\sqrt{N}}$  and  plugging bound of $\norm{ \bar\theta_T - \hat{\theta}  }$  yields:
    \begin{align*}
        &f(\hat{\theta} )- f(\theta^* )  \\
          &\leq \pare{ G  +  \frac{ 2C_g  \sqrt{3\log\frac{2}{\delta}}}{c_\eta G^2 + \frac{1}{c_\eta}} G}\pare{  \frac{2  G C_g   \sqrt{3\log\frac{2}{\delta}} + 6  G (c_\eta G^2 + \frac{1}{c_\eta}) }{\sqrt{N}(r-2H\epsilon_0)^2 }      +  \frac{\Delta}{2 (r-2H\epsilon_0) } } \\
         &  +  \pare{1+ G  +  \frac{  C_g  \sqrt{ \log\frac{2}{\delta}}}{c_\eta G^2 + \frac{1}{c_\eta}} G}\pare{ \frac{\frac{\rho^2}{c_{\eta}} + c_{\eta}(  C_g^2 + G^2) + 2C_g \sqrt{ \log \frac{2}{\delta}} \pare{\lambda^* +  \rho  }+ \rho G\sqrt{ \log\frac{2}{\delta}}}{\sqrt{N}}}.
         \end{align*}

Now we simplify the above bound. By the definition of $c_{\eta}$ we know $c_{\eta} \leq  \frac{1}{ G}$  , so we have

 \begin{align*}
        &f(\hat{\theta} )- f(\theta^* )  \\
          &\leq \pare{2 G  +  \frac{ 2C_g  \sqrt{3\log\frac{2}{\delta}}}{c_\eta G  +1}  }\pare{  \frac{2 G C_g   \sqrt{3\log\frac{2}{\delta}} + 6  G (c_\eta G^2 + \frac{1}{c_\eta}) }{\sqrt{N}(r-2H\epsilon_0)^2 }      +  \frac{\Delta}{2 (r-2H\epsilon_0) } } \\
         &  +  \pare{2 G  +  \frac{  2C_g  \sqrt{3 \log\frac{2}{\delta}}}{c_\eta  G  + 1}  }\pare{ \frac{\frac{\rho^2}{c_{\eta}} + c_{\eta}(  C_g^2 + G^2) + (C_g   \pare{\lambda^* +  \rho  }+ \rho G)\sqrt{ \log\frac{2}{\delta}}}{\sqrt{N}}}.
         \end{align*}

    Finally, by definition of $c_{\eta}$ we know $\frac{\rho^2}{c_{\eta} }  \geq c_{\eta} C_g^2$, $\frac{\rho^2}{c_{\eta}} \geq \rho C_g$, $\frac{\rho^2}{c_{\eta}} \geq \rho { 12\sqrt{6}    G \sqrt{ \log\frac{2}{\delta}} }  $ and $\frac{G/c_{\eta}}{(r-2H\epsilon_0)^2} \geq c_\eta G^2 $  which concludes the proof:

\begin{align*}
        &f(\hat{\theta} )- f(\theta^* )  \lesssim \pare{ G  +     C_g  \sqrt{ \log\frac{1}{\delta}}  }\pare{  \frac{  G C_g   \sqrt{ \log\frac{1}{\delta}} +     \frac{ G}{c_\eta}  }{\sqrt{N}(r-2H\epsilon_0)^2 }      +  \frac{\Delta}{  (r-2H\epsilon_0) } + \frac{\frac{\rho^2}{c_{\eta}}  +\lambda^*  C_g   \sqrt{ \log\frac{1}{\delta}}}{\sqrt{N}} } 
         \end{align*}
Plugging the bound that $\lambda^*\leq G/r$ (see~\cite[Remark~1]{mahdavi2012stochastic}) will conclude the proof.   
     \end{proof}

\subsection{Proof of Theorem~\ref{thm:opt rate}}
\begin{proof}

Now we proceed to proving Theorem~\ref{thm:opt rate}. We split the proof into two parts: proof of optimization guarantee (the gradient complexity) and proof of statistical rate.
\paragraph{Convergence rate}
The total gradient complexity of Algorithm~\ref{algorithm: hat theta} will be the sum of the complexities of each calling of procedure $\CPSolver$. Theorem~\ref{thm:cpsolver rate} gives the complexity of $\CPSolver$, and what we need to compute is $\epsilon_0$ in each call. Recall that, $\epsilon_0$ is the error between the true constraint and the approximate we actually used in $\CPSolver$. Hence, we first compute those error as follows:
\begin{align*}
    |\hat{g}'(\theta,\alpha') -  \hat{g}'(\theta,\alpha')| &\leq |\hat g_{1,T}(\theta) -  g_{1,T}(\theta) |\leq \hat{R}_{\varphi, \mu_{1,T}}(\hat{\theta}_{T,\alpha-\epsilon_{0,T}}) - \hat{R}_{\varphi, \mu_{1,T}}(\hat{\theta}^*_{T,\alpha-\epsilon_{0,T}}) \\
    &\leq \epsilon_{T,\alpha-\epsilon_{0,T}},\\
    |\hat{g}_{\hat{\alpha}}(\theta ) -  {g}_{\hat{\alpha}_S}(\theta )| &\leq  \max\cbr{|\hat\alpha - \hat{\alpha}_S|,|\hat g_{1,T}(\theta) -  g_{1,T}(\theta) |} \leq \max\cbr{ \epsilon_{\hat{\alpha}_S},  \epsilon_{T,\alpha-\epsilon_{0,T}}},\\
    |\hat{g}_{S,T}(\theta) - g_{S,T}(\theta)|&\leq \max\cbr{ |\hat{g}_{\hat{\alpha}}(\theta ) -  {g}_{\hat{\alpha}_S}(\theta )|, |\hat g'_T(\theta) - g'_T(\theta)| }\\
   & \leq  \max\cbr{ \epsilon_{\hat{\alpha}_S},  \epsilon_{T,\alpha-\epsilon_{0,T}}, \hat{R}_{1,T}(\hat\theta^*_{T,\hat\alpha_S}) - \hat{R}_{1,T}(\hat\theta_{T,\hat\alpha_S})   } \\
   &\leq  \max\cbr{ \epsilon_{\hat{\alpha}_S},  \epsilon_{T,\alpha-\epsilon_{0,T}}, \epsilon'_{T} }.
\end{align*}

Now we verify that the choice of the tolerance in Algorithm~\ref{algorithm: hat theta} can ensure that $$\hat{R}_{\varphi,\mu_{1,S}}(\tilde \theta) - \min_{\theta: g_{S,T}(\theta)\leq 0} \hat{R}_{\varphi,\mu_{1,S}}( \theta) \leq \epsilon_{1,S}=\epsilon_{S,T}.$$
To ensure $\CPSolver\!\left(
   \hat R_{\varphi,\mu_{1,S}},\ \hat g_{S,T},\ \xi(\epsilon_{S,T}, r_{S,T}),\ \epsilon_{S,T}
 \right)$ outputs $\epsilon_{S,T}$-accurate solution, by the condition in Theorem~\ref{thm:cpsolver rate}, we need $|\hat{g}_{S,T}(\theta) - g_{S,T}(\theta)| \leq \xi(\epsilon_{S,T}, r_{S,T})$, hence we require  
 \begin{align}
      \epsilon_{\hat{\alpha}_S} \leq \xi(\epsilon_{S,T}, r_{S,T}) \label{eq:alpha bound1}
      , \\
      \epsilon_{T,\alpha-\epsilon_{0,T}}\leq \xi(\epsilon_{S,T}, r_{S,T}), \label{eq:T bound1}\\
      \epsilon'_{T}  \leq \xi(\epsilon_{S,T}, r_{S,T})
 \end{align}
 We also require $\epsilon'_T \leq \epsilon_{1,T}$, so we know choosing $\epsilon'_{T} = \min\cbr{\xi(\epsilon_{S,T}, r_{S,T}),  \epsilon_{1,T}}$ suffices.

 To ensure $\CPSolver\!\left(
   \hat R_{\varphi,\mu_{1,T}},\ \hat g_{\hat{\alpha}},\ \xi(\epsilon'_T, r_T'),\ \epsilon'_T
 \right)$ output $\epsilon'_T$-accurate solution, we need $|\hat{g}_{\hat{\alpha}}(\theta ) -  {g}_{\hat{\alpha}_S}(\theta )| \leq \xi(\epsilon'_T, r'_{T})$, hence we require 
 \begin{align}
     \epsilon_{\hat{\alpha}_S}\leq \xi(\epsilon'_T, r'_{T}),\label{eq:alpha bound2}\\
     \epsilon_{T,\alpha-\epsilon_{0,T}}\leq \xi(\epsilon'_T, r'_{T}) \label{eq:T bound2} 
 \end{align}
 
 To ensure $\CPSolver\!\left(
   \hat R_{\varphi,\mu_{1,S}},\ \hat g_{\hat{\alpha}},\ \xi(\epsilon'_S, r_S'),\ \epsilon'_S
 \right)$ output $\epsilon'_S = \epsilon_{1,S}$-accurate solution, we need $|\hat{g}_{\hat{\alpha}}(\theta ) -  {g}_{\hat{\alpha}_S}(\theta )| \leq \xi(\epsilon'_S, r'_{S})$, hence we require 
 \begin{align}
     \epsilon_{\hat{\alpha}_S}\leq \xi(\epsilon'_S, r'_{S}),\label{eq:alpha bound3}\\
     \epsilon_{T,\alpha-\epsilon_{0,T}}\leq \xi(\epsilon'_S, r'_{S})  \label{eq:T bound3}
 \end{align} 

 To ensure $\CPSolver\!\left(
  \alpha',\ \hat g'(\theta,\alpha'),\ \xi(\epsilon_{\hat\alpha_S}, r_{\hat\alpha_S}),\ \epsilon_{\hat\alpha_S}
\right)$ output $\epsilon_{\hat{\alpha}_S}$-accurate solution, we need $|\hat{g}_{1,T}(\theta ) -  {g}_{1,T}(\theta )| \leq \xi(\epsilon_{\hat\alpha_S}, r_{\hat\alpha_S})$, hence we require 
 \begin{align} 
     \epsilon_{T,\alpha-\epsilon_{0,T}}\leq \xi(\epsilon_{\hat\alpha_S}, r_{\hat\alpha_S}) \label{eq:T bound4} 
 \end{align} 

 From~\eqref{eq:alpha bound1},~\eqref{eq:alpha bound2} and~\eqref{eq:alpha bound3} we know that choosing $\epsilon_{\hat{\alpha}_S} \leq \min\cbr{\xi(\epsilon'_{S,T}, r'_{S,T}), \xi(\epsilon'_T, r'_{T}),\xi(\epsilon'_S, r'_{S}) }$ suffices and from~\eqref{eq:T bound1},~\eqref{eq:T bound2},~\eqref{eq:alpha bound3} and~\eqref{eq:T bound4} we know that choosing $\epsilon_{T,\alpha-\epsilon_{0,T}} \leq \min\cbr{\xi(\epsilon'_{S,T}, r'_{S,T}), \xi(\epsilon'_T, r'_{T}),\xi(\epsilon'_S, r'_{S}),\xi(\epsilon_{\hat\alpha_S}, r_{\hat\alpha_S}) }$ suffices.

 The complexity immediately follows by plugging error tolerances into Theorem~\ref{thm:cpsolver rate}.

\paragraph{Statistical Rate}
If $\hat R_{\varphi,\mu_{1,S}}(\tilde\theta) - \hat R_{\varphi,\mu_{1,S}}(\hat{\theta}'_S) > 2\epsilon_{1,S}$, then we know 
\begin{align*}
  \min_{\theta: g_{S,T}(\theta)\leq 0}  \hat R_{\varphi,\mu_{1,S}}( \theta) - \min_{\theta:g_{\hat{\alpha}_S}\leq \epsilon_{1,S}}\hat R_{\varphi,\mu_{1,S}}(\theta)  \geq \hat R_{\varphi,\mu_{1,S}}(\tilde\theta) - \hat R_{\varphi,\mu_{1,S}}(\hat{\theta}'_S)-\epsilon_{1,S} \geq \epsilon_{1,S}
\end{align*}

Hence the set $\cbr{\theta: g_{\hat{\alpha}_S}(\theta) \leq 0, g_T'(\theta)\leq 0}$   does not intersect with $\cbr{\theta: g_{\hat{\alpha}_S}(\theta) \leq 0, g_S'(\theta)\leq \epsilon_{1,S}}$ which is $ \{h\in\hat{\mathcal{H}}':\hat{R}_{\varphi,\mu_{1,S}}(h)\leq \hat{R}^*_{\varphi,\mu_{1,S}}(\hat{\mathcal{H}}')+ 3\epsilon_{1,S}\} $ (a slightly inflated version of $\hat{\mathcal{H}}'_{1,S}$)  , so we output $\hat{\theta}'_T$.

Otherwise, $\hat R_{\varphi,\mu_{1,S}}(\hat\theta) - \hat R_{\varphi,\mu_{1,S}}(\hat{\theta}'_S) \leq 2\epsilon_{1,S}$, then 
\begin{align*}
  \min_{\theta: g_{S,T}(\theta)\leq 0}  \hat R_{\varphi,\mu_{1,S}}( \theta)  - \min_{\theta:g_{\hat{\alpha}_S}\leq \epsilon_{1,S}}\hat R_{\varphi,\mu_{1,S}}(\theta) \leq  \hat R_{\varphi,\mu_{1,S}}( \tilde \theta)  - \hat R_{\varphi,\mu_{1,S}}(\hat{\theta}'_S) - \epsilon_{1,S} \leq  \epsilon_{1,S},
\end{align*} 
which implies that the set $\cbr{\theta: g_{\hat{\alpha}_S}(\theta) \leq 0, g_T'(\theta)\leq 0}$ ($\hat{\mathcal{H}}'_T$) intersects with $\cbr{\theta: g_{\hat{\alpha}_S}(\theta) \leq 0, g_S'(\theta)\leq - \epsilon_{1,s}}$ which is $ \{h\in\hat{\mathcal{H}}':\hat{R}_{\varphi,\mu_{1,S}}(h)\leq \hat{R}^*_{\varphi,\mu_{1,S}}(\hat{\mathcal{H}}')+  \epsilon_{1,S}\} $ (a slightly shrinked version of $\hat{\mathcal{H}}'_{1,S}$  ), so we output a model $\tilde \theta$ in their intersection. 

Hence the output of our optimization procedure is consistent with the learning algorithm~\eqref{procedure_intersection}, and the same generalization analysis applies here.
\end{proof}

\end{document}